\documentclass{article}

\usepackage{cite}

\usepackage{amsmath}
\usepackage[final]{pdfpages}

\usepackage{amsthm}
\usepackage{amssymb}
\usepackage{titling}
\usepackage[margin=1.2in]{geometry}
\newtheorem{theorem}{Theorem}
\newtheorem{lemma}[theorem]{Lemma}

\newtheorem{definition}[theorem]{Definition}
\newtheorem{algorithm}[theorem]{Algorithm}
\newtheorem{structure}[theorem]{Data-Structure}
\newtheorem{operation}[theorem]{Operation}

\newcommand{\JTvertex}[1]{\mathcal{V}(#1)}
\newcommand{\JT}{\mathcal{J}}
\newcommand{\JTroot}{R}
\newcommand{\JTpar}[1]{{\uparrow}(#1)}
\newcommand{\JTmess}[2]{M_{#1\rightarrow#2}}
\newcommand{\JTfactor}[1]{\mathcal{F}(#1)}

\newcommand{\JTmarg}[2]{{#1}^{{\bigtriangledown}#2}}

\newcommand{\JTneighbour}[1]{\mathcal{N}(#1)}
\newcommand{\JTresult}[1]{\rho_{#1}}
\newcommand{\JTpower}[1]{\mathcal{P}(#1)}
\newcommand{\JTnatup}[1]{\mathbb{N}_{#1}}
\newcommand{\JTpots}[1]{\mathcal{T}(#1)}
\newcommand{\JTexp}[1]{2^{#1}}
\newcommand{\JTunderlying}[1]{\sigma(#1)}

\newcommand{\JTdegree}[1]{\operatorname{deg}(#1)}
\newcommand{\JTdual}[1]{\##1}
\newcommand{\JTmdual}[1]{\%#1}
\newcommand{\JTnumb}[2]{\kappa(#1,#2)}
\newcommand{\JTperm}[2]{#2^{(-1)^{#1}}}

\newcommand{\JTedge}[1]{\mathcal{E}(#1)}
\newcommand{\JTchildren}[1]{{\downarrow}(#1)}
\newcommand{\JTcontains}[1]{{#1}^{+}}
\newcommand{\JTsp}[2]{{#1}^{\bullet{#2}}}
\newcommand{\JTbase}{S}
\newcommand{\JTone}[1]{\boldsymbol{1}_{#1}}

\newcommand{\JTstree}[1]{\mathfrak{B}(#1)}
\newcommand{\JTitreeone}[1]{\mathfrak{T}(#1)}
\newcommand{\JTitreetwo}[2]{\mathfrak{T}(#1, #2)}
\newcommand{\JTlabel}[1]{\phi(#1)}
\newcommand{\JTvalue}[1]{\psi(#1)}
\newcommand{\JTlchild}[1]{{\triangleleft}(#1)}
\newcommand{\JTrchild}[1]{{\triangleright}(#1)}
\newcommand{\JTleafset}[1]{\tau(#1)}
\newcommand{\JTancestor}[1]{{\Uparrow}(#1)}
\newcommand{\JTdescendants}[1]{{\Downarrow}(#1)}
\newcommand{\JTstack}{\mathfrak{Z}}
\newcommand{\JTnull}{0}
\newcommand{\JTarray}[1]{{\mathcal{A}}(#1)}
\newcommand{\JTL}{\mathcal{L}}
\newcommand{\JTmultiset}{\mathfrak{X}}

\title{A Time and Space Efficient Junction Tree Architecture}
\author{Stephen Pasteris\\
Department of Computer Science\\
University College London\\
London
WC1E 6BT, England, UK\\
s.pasteris@cs.ucl.ac.uk}
\begin{document}
\maketitle

\begin{abstract}
The junction tree algorithm is a way of computing marginals of boolean multivariate probability distributions that factorise over sets of random variables. 
The junction tree algorithm first constructs a tree called a junction tree who's vertices are sets of random variables. The algorithm then performs a generalised version of belief propagation on the junction tree.
The Shafer-Shenoy and Hugin architectures are two 
ways to perform this belief propagation that tradeoff time and space complexities in different ways: Hugin propagation is at least as fast as Shafer-Shenoy propagation and in the cases that we have large vertices of high degree is significantly faster. However, this speed increase comes at the cost of an increased space complexity. This paper first introduces a simple novel architecture, ARCH-1, which has the best of both worlds: the speed of Hugin propagation and the low space requirements of Shafer-Shenoy propagation. A more complicated novel architecture, ARCH-2, is then introduced which has, up to a factor only linear in the maximum cardinality of any vertex, time and space complexities at least as good as ARCH-1 and in the cases that we have large vertices of high degree is significantly faster than ARCH-1. 
\end{abstract}

\section{Introduction}
The junction tree algorithm is a popular tool for the simultaneous computation of all marginals of a multivariate probability distribution stored in a factored form. In this paper we consider the case in which the random variables are boolean. The junction tree algorithm is a generalisation of belief propagation \cite{JTref2} performed on a tree (called a junction tree) who's vertices are sets of random variables. The Shafer-Shenoy \cite{JTref3} \cite{JTref1} and Hugin \cite{JTref10} \cite{JTref1} architectures are two variations of the junction tree algorithm that trade off time and space complexities in different ways: Hugin propagation is faster than Shafer-Shenoy propagation but at the cost of a greater space complexity. Large vertices of high degree cause much inefficiency in both these architectures (especially in that of Shafer-Shenoy) and it is the purpose of this paper to introduce novel architectures that perform better in these cases. In order to tackle the problem of high-degree vertices, an algorithm was given in \cite{JTref6} for constructing a binary junction tree 
 on which Shafer-Shenoy propagation can then be performed. This method was shown empirically to be faster than Hugin propagation (on a generic junction tree constructed in a certain way) in \cite{JTref4}. The drawback of this method, however, is that it can require dramatically more space than Shafer-Sheony propagation on a generic junction tree due to the maximum cardinality of the intersection of two neighbouring vertices being large. It should also be noted that in \cite{JTref5} an architecture was given that eliminated the redundant computations (caused by high-degree vertices) of Shafer-Shenoy propagation. Again though, this architecture can have a dramatically increased space complexity over that of Shafer-Shenoy propagation. In comparison, the architectures introduced in this paper tackle the problem of high degree vertices whilst retaining the low space complexity of Shafer-Shenoy propagation (on a generic junction tree). Two novel architectures are introduced in this paper: the first, ARCH-1, achieves the speed (up to a constant factor) of Hugin propagation and has the low space requirements of Shafer-Shenoy propagation. ARCH-1 is very simple and serves as a warm up to a more complicated architecture, ARCH-2, which (almost) has space and time complexities at least as good as ARCH-1 and in the cases in which we have large vertices of high degree is significantly faster than ARCH-1/Hugin. In the cases in which we have a large enough (relative to the rest of the junction tree)
 vertex who's degree is exponential (to some base greater than one) in its cardinality then ARCH-2 has a polynomial saving in the time complexity over that of ARCH-1/Hugin: i.e. there exists $s<1$ such that a time of $\Theta(t)$ (for ARCH-1/Hugin) becomes a time of $\mathcal{O}(t^s)$ (for ARCH-2). The saving in time complexity in going from Shafer-Shenoy to Hugin/ARCH-1 is similar.  \\A more detailed description of the results of this paper is given in section \ref{JTsection3} after the preliminary definitions have been introduced and the junction tree algorithm has been described.
 \newline
 \newline
In this paper we assume that all basic operations such as arithmetic operations and memory reads/writes take constant time. To ease the reader's understanding, the algorithms given in sections \ref{JTsection4} though \ref{JTsection6} are sketches: to achieve the stated time complexities we must be able to find and store variables in constant amortised time and space. The exact implementations that give the stated time and space complexities are given in Section \ref{JTsection8}.
\newline
\newline
This paper is structured as follows: In Section \ref{JTsection2} we give the preliminary definitions required by the paper. In Section \ref{JTsection3} we give an overview of the junction tree algorithm, a detailed overview of the results of the paper and some technicalities relating to time/space complexities. In Section \ref{JTsection4} we describe the Shafer-Shenoy and Hugin architectures and analyse their complexities. In Section \ref{JTsection5} we describe the architecture ARCH-1 and analyse its complexity. In Section \ref{JTsection6} we describe the architecture ARCH-2 and analyse its complexity. In Section \ref{JTsection7} we describe how to modify ARCH-2 such that it can deal with zeros. In Section \ref{JTsection8} we give the details of how to implement the algorithms introduced in this paper.

\noindent\hfil\rule{1\textwidth}{.4pt}\hfil

\section{Preliminaries}
\label{JTsection2}
In this section we define the notation and concepts used in this paper except for that required exclusively for the implementation details of the algorithms which are defined in Section \ref{JTsection8}. Also, the notation $\JT$, $S$, $\JTfactor{C}$ and $\JTmess{H}{C}$, as well as the notion of ``sending" and ``receiving" messages, is defined in Algorithm \ref{JTA}.

\noindent\hfil\rule{1\textwidth}{.4pt}\hfil
\subsection{Basic Notation}
The symbol $:=$ is used for definition: e.g. $x:=y$ means ``$x$ is defined to be equal to $y$". Given $a\in\mathbb{N}$ we define $\JTnatup{a}$ to be equal to the set of the first $a$ natural numbers: i.e. the set $\{1, 2, 3, ..., (a-1), a\}$. Given a set $X$ we define $\JTpower{X}$ to be the \textit{power-set} of $X$: that is, the set of all subsets of $X$.
\newline
\newline
We now define the pseudo-code used in this paper: The left arrow, $\leftarrow$, denotes assignment: e.g. $a\leftarrow b$ indicates that the value $b$ is computed and then assigned to the variable $a$. Function names are written in bold with the input coming in brackets after the name. When the assignment symbol, $\leftarrow$, has a function on its right hand side it indicates that the function is run and the output of the function is assigned to the variable on the left hand side: e.g. $a\leftarrow{\bf function}(b)$ indicates that the function {\bf function} is run with input $b$ and its output is assigned to variable $a$. When the word ``return" appears in the pseudo-code for a function it indicates that the function terminates and outputs the object coming after the word ``return".

\noindent\hfil\rule{1\textwidth}{.4pt}\hfil

\subsection{Potentials}
A \textit{binary labelling} of a set $X$ is a map from $X$ into $\{0,1\}$.
\newline
\newline
A \textit{potential} on a set $X$ is a map from $\JTpower{X}$ into $\mathbb{R}^+$.\\ 
Given a potential $\Psi$ on a set $X$ we define $\JTunderlying{\Psi}:=X$.\\ Given a set $X$ we define $\JTpots{X}$ to be the set of all possible potentials on $X$.\\
Given a set $X$ we define $\JTone{X}$ to be the potential in $\JTpots{X}$ that satisfies $\JTone{X}(Z):=1$  for all $Z\in\JTpower{X}$.
\newline
\newline
Note that a potential on a set $X$ is equivalent to a map from all possible binary labellings of $X$ into the positive reals (which is the usual definition of a potential). The equivalence is seen by noting that there is a bijecitive mapping from $\JTpower{X}$ into the set of all possible binary labellings of $X$ where a subset $Y$ of $X$ maps to the labelling $\mu_Y$ of $X$ given by $\mu_Y(x):=1$ for all $x\in Y$ and $\mu_Y(x):=0$ for all $x\in X\setminus Y$. The operations in this paper are easier to describe when the domain of a potential is a power-set, which is why we define potentials in this way.

\noindent\hfil\rule{1\textwidth}{.4pt}\hfil

Given a potential $\Psi$ on a set $X$ and a subset $Y\subseteq X$ we define the \textit{$Y$-marginal}, $\JTmarg{\Psi}{Y}$, of $\Psi$ as the potential in $\JTpots{Y}$ that satisfies, for all $Z\in\JTpower{Y}$:
\begin{equation}
\JTmarg{\Psi}{Y}(Z):=\sum_{U\in\JTpower{X}:U\cap Y=Z}\Psi(U)
\end{equation}
Note that, by above, $\Psi$ may be equivalent to a probability distribution on binary labellings of $X$. If this is the case then $\JTmarg{\Psi}{Y}$ is equivalent to the marginalisation of that probability distribution onto $Y$.

\noindent\hfil\rule{1\textwidth}{.4pt}\hfil

Given sets $X$ and $Y$ and potentials $\Psi\in\JTpots{X}$ and $\Phi\in\JTpots{Y}$ we define the \textit{product}, $\Psi\Phi$, of $\Psi$ and $\Phi$ as the potential in $\JTpots{X\cup Y}$ that satisfies, for all $Z\in\JTpower{X\cup Y}$:
\begin{equation}
[\Psi\Phi](Z):=\Psi(Z\cap X)\Phi(Z\cap Y)
\end{equation}
We represent the product of multiple potentials by the $\prod$ symbol, as in the multiplication of numbers.

\noindent\hfil\rule{1\textwidth}{.4pt}\hfil

Given a set $X$ and potentials $\Psi, \Phi\in\JTpots{X}$ we define the \textit{quotient}, $\Psi/\Phi$, of $\Psi$ and $\Phi$ as the potential in $\JTpots{X}$ that satisfies, for all $Z\in\JTpower{X}$:
\begin{equation}
[\Psi/\Phi](Z):=\Psi(Z)/\Phi(Z)
\end{equation}

\noindent\hfil\rule{1\textwidth}{.4pt}\hfil

The reason for the low space complexity of ARCH-2 is that, given a set $X$ and a potential $\Phi\in\JTpots{X}$, we may not need to store the value of $\Phi(Y)$ for every $Y\in\JTpower{X}$. This encourages the following definitions:

A set $\zeta$ of sets is a \textit{straddle-set} if and only if, for every $Z\in\zeta$ and every $Y\in\JTpower{Z}$ we have $Y\in\zeta$.

Given a potential $\Phi$ and a straddle-set $\zeta\subseteq\JTpower{\JTunderlying{\Phi}}$, the \textit{sparse format}, $\JTsp{\Phi}{\zeta}$, is the data-structure that stores the value $\Phi(Y)$ if and only if $Y\in\zeta$.

Note that storing the sparse format $\JTsp{\Phi}{\zeta}$ requires a space of only $\Theta(|\zeta|)$.

\noindent\hfil\rule{1\textwidth}{.4pt}\hfil

ARCH-2 works with the notions of the \textit{p-dual}, $\JTdual{\Psi}$, and the \textit{m-dual}, $\JTmdual{\Psi}$, of a potential $\Psi$. These are defined in sections \ref{pdualsection} and \ref{mdualsection} respectively. 

\noindent\hfil\rule{1\textwidth}{.4pt}\hfil

\subsection{Factorisations}
Suppose we have a probability distribution $\mathbb{P}$ on the set of binary labellings of a set $\JTbase$. Then a set, $\mathcal{F}$, of potentials is a \textit{factorisation} of $\mathbb{P}$ if and only if $\bigcup_{\Lambda\in\mathcal{F}}\JTunderlying{\Lambda}=S$ and for every binary labelling, $\mu$ of $\JTbase$ we have:
\begin{equation}
\mathbb{P}(\mu)\propto\left[\prod_{\Lambda\in\mathcal{F}}\Lambda\right]\left(\{x\in\JTbase:\mu(x)=1\}\right)
\end{equation}

\noindent\hfil\rule{1\textwidth}{.4pt}\hfil

\subsection{Junction Trees}
Given a tree $\JT$ we define $\JTvertex{\JT}$ and $\JTedge{\JT}$ to be the vertex and edge set of $\JT$ respectively. Also, given a tree $\JT$ and a vertex $ C\in\JTvertex{\JT}$ we define $\JTdegree{ C}$ and $\JTneighbour{ C}$ to be the degree (i.e. number of neighbours) and neighbourhood (i.e. set of neighbours) of $ C$ in $\JT$ respectively. When a tree $\JT$ is rooted we define, for a vertex $C\in\JTvertex{\JT}$, $\JTpar{C}$ and $\JTchildren{C}$ to be the parent of $C$ and the set of children of $C$  respectively.
\newline
\newline
A \textit{junction tree}, $\JT$, on a set $S$ is a tree satisfying the following axioms:
\begin{itemize}
\item Every vertex of $\JT$ is a subset of $S$. 
\item $\bigcup\JTvertex{\JT}=S$
\item Given $ C,  H\in\JTvertex{\JT}$ and some $x\in S$ such that $x\in C\cap H$ then $x$ is a member of every vertex in the path (in $\JT$) from $ C$ to $ H$.
\end{itemize}
The \textit{width} of a junction tree is defined as the cardinality of its largest vertex.


\noindent\hfil\rule{1\textwidth}{.4pt}\hfil

\section{The Junction Tree Algorithm}
\label{JTsection3}

The goal of this paper is as follows: We have a probability distribution $\mathbb{P}$ on binary labellings, $\mu$, of a set $\JTbase$ and a factorisation, $\mathcal{F}$, of $\mathbb{P}$. We wish  to compute the marginal probability $\mathbb{P}(\mu(x)=1)$ for every $x\in S$. 
\newline
\newline
Note that the marginal $\mathbb{P}(\mu(x)=1)$ is equivalent to a potential $\JTresult{x}\in\JTpots{x}$ defined as $\JTresult{x}(\emptyset):=\mathbb{P}(\mu(x)=0)$ and $\JTresult{x}(\{x\}):=\mathbb{P}(\mu(x)=1)$.
\newline
\newline
The \textit{junction tree algorithm} is a way of simultaneously computing the potentials $\JTresult{x}$ for every $x\in S$. The algorithm has three stages: The \textit{junction tree construction stage}, the \textit{message passing stage} and the \textit{computation of marginals stage}:

\begin{algorithm}
\label{JTA}
{\bf The Junction Tree Algorithm:}
\begin{enumerate}
\item Junction tree construction stage:\\ A junction tree $\JT$ on $S$ is constructed such that for all $\Lambda\in\mathcal{F}$ we have a vertex $\JTcontains{\Lambda}\in\JTvertex{\JT}$ for which $\JTunderlying{\Lambda}\subseteq\JTcontains{\Lambda}$. For every $ C\in\JTvertex{\JT}$ we define $\JTfactor{ C}:=\{\Lambda\in\mathcal{F}:\JTcontains{\Lambda}= C\}$.
\item \label{JTmessagepassing} Message passing stage:\\ For every ordered pair $( C,  E)$ of neighbouring vertices of $\JT$, we create and store a message $\JTmess{ C}{ E}$ which is a potential in $\JTpots{ C\cap E}$. When such a message is created we say that $ C$ ``sends" the message and $ E$ ``receives" the message. The messages are defined recursively by the following equation:
\begin{equation} \label{JTmultimess}
\JTmess{ C}{ E}:=\JTmarg{\left[\left(\prod_{\Lambda\in\JTfactor{ C}}\Lambda\right)\left(\JTone{ C\cap E}\prod_{ H\in\JTneighbour{ C}\setminus\{ E\}}\JTmess{ H}{ C}\right)\right]}{ C\cap E}
\end{equation}
\item Computation of marginals stage:\\ For every $x\in S$ we compute the potential $\JTresult{x}$ from the messages. Specifically, for any vertex $ C\in\JTvertex{\JT}$ with $x\in C$ we have:
\begin{equation}
\JTresult{x}=\JTmarg{\left[\left(\prod_{\Lambda\in\JTfactor{ C}}\Lambda\right)\left(\prod_{ H\in\JTneighbour{ C}}\JTmess{ H}{ C}\right)\right]}{\{x\}}
\end{equation}
\end{enumerate}
\end{algorithm}

\noindent\hfil\rule{1\textwidth}{.4pt}\hfil
In the rest of the paper the symbols $\JT$, $S$, $\JTfactor{ C}$ and $\JTmess{ H}{ C}$, as well as the notion of ``sending" and ``receiving" messages, are always defined as above.

\noindent\hfil\rule{1\textwidth}{.4pt}\hfil
In this paper we consider, in detail, the message passing stage of the junction tree algorithm:
\newline
We first review the Shafer-Shenoy and Hugin architectures that differ in how the messages are computed. With Shafer-Shenoy propagation each vertex $ C$ contributes a time of\\ $\Theta\left(\JTdegree{ C}(\JTdegree{ C}+|\JTfactor{ C}|)\JTexp{| C|}\right))$ to the message passing stage whilst with  Hugin propagation each vertex $ C$ contributes a time of $\Theta\left((\JTdegree{ C}+|\JTfactor{ C}|)\JTexp{| C|}\right)$ to the message passing stage. When we have large vertices of high degree Hugin propagation is hence significantly faster than Shafer-Shenoy propagation. However, this speed increase comes at a cost of a higher space complexity: whilst the space complexity of Shafer-Shenoy propagation is only that required to store the factors and messages, the Hugin architecture must store, for every vertex $ C$, a potential $\Psi_{ C}\in\JTpots{ C}$; meaning that the space required is exponential (base $2$) in the width of the junction tree.\\
 We then describe, from a merger of the ideas behind Shafer-Shenoy and Hugin propagation, a simple, novel architecture ARCH-1 which has (up to a constant factor)
  the best of both worlds: the speed of Hugin propagation and the low space complexity of Shafer-Shenoy propagation.\\
The main idea behind ARCH-1, that of simultaneously computing many marginals of a factored potential, then leads us into the novel architecture ARCH-2 which has (up to a factor linear in the width of the junction tree) at least the time and space efficiency of ARCH-1 and is considerably faster when we have large vertices of high degree. Specifically, each vertex $ C$ now contributes a time of only $\mathcal{O}\left(| C|\JTexp{| C|}\right)$ to the message passing stage and, in addition to storing the factors and messages, ARCH-2 requires a space of only $\mathcal{O}\left(\operatorname{max}_{ C\in\JTvertex{\JT}}| C|\left(\left(\sum_{ H\in\JTneighbour{ C}}\JTexp{| H\cap C|}\right)+\left(\sum_{\Lambda\in\JTfactor{ C}}\JTexp{|\JTunderlying{ \Lambda}|}\right)\right)\right)$.
\newline
\newline
We note that, although we don't explicitly describe the computation of marginals stage, the ideas behind ARCH-2 can be used to do this stage with time and space no greater than the message passing stage of ARCH-2 
. The details are left to the reader.
\newline
\newline
As stated in the introduction, to ease the reader's understanding, the algorithms given in sections \ref{JTsection4} though \ref{JTsection6} are sketches: to achieve the stated time complexities we must be able to find and store variables in constant amortised time and space. The exact implementations that give the stated time and space complexities are given in Section \ref{JTsection8}.
\newline
\newline
We also note, that the auxiliary space required by the algorithms in this paper is an additive factor of $\mathcal{O}(|S|)$ more than is stated since we must maintain an array of size $|S|$ (see Section \ref{JTsection8}). However, since $\mathcal{O}(|S|)$ is no greater than the space required to store the factors it is fine to neglect this.

\noindent\hfil\rule{1\textwidth}{.4pt}\hfil

\section{Shafer-Shenoy and Hugin Propagation}
\label{JTsection4}

\subsection{Shafer-Shenoy Propagation}
In this subsection we describe and analyse the complexity of Shafer-Shenoy propagation. Shafer-Shenoy propagation follows the following algorithm:


\begin{algorithm} {\bf Outline of Shafer-Shenoy Propagation:}\\
Given an ordered pair $( C,  E)$ of neighbouring vertices, once $ C$ has received messages from all vertices in $\JTneighbour{ C}\setminus\{ E\}$ the message $\JTmess{ C}{ E}$ is computed as:
\begin{equation}
\JTmess{ C}{ E}\leftarrow\JTmarg{\left[\left(\prod_{\Lambda\in\JTfactor{ C}}\Lambda\right)\left(\JTone{ C\cap E}\prod_{ H\in\JTneighbour{ C}\setminus\{ E\}}\JTmess{ H}{ C}\right)\right]}{ C\cap E}
\end{equation}
and is sent from $ C$ to $ E$.
\end{algorithm}

\noindent\hfil\rule{1\textwidth}{.4pt}\hfil

Note that the creation of a message in the above algorithm is an instance of the following operation (where $\{D_1, D_2, ..., D_k\}:=\{\JTunderlying{\Lambda}:\Lambda\in\JTfactor{ C}\}\cup\{ C\cap H: H\in\JTneighbour{ C}\}$):

\begin{operation}
\label{JTop1}
We have a set $C$, subsets $\{D_1, D_2, ... , D_k\}\subseteq\JTpower{C}$ and a subset $W\subseteq C$. For every $i\in\JTnatup{k}$ we have a potential $\Upsilon_i\in\JTpots{D_i}$. We must compute $\JTmarg{\left(\prod_{i=1}^k \Upsilon_i\right)}{W}$.
\end{operation}

\noindent\hfil\rule{1\textwidth}{.4pt}\hfil

If operation \ref{JTop1} is performed by firstly computing $\prod_{i=1}^k \Upsilon_i$ and then marginalising it onto $W$ it requires an auxiliary space on $\Theta\left(\JTexp{|C|}\right)$ leading to a space requirement of at least $\Omega\left(\operatorname{max}_{ H\in\JTvertex{\JT}}\JTexp{| H|}\right)$ for the whole algorithm. Hence, we now give an algorithm that can be implemented to perform operation \ref{JTop1} in a time of $\Theta\left(k\JTexp{|C|}\right)$ and which uses only constant auxiliary space:

\begin{algorithm}
\label{JTalg1}
For every $Y\in \JTpower{W}$ we maintain a variable $h(Y)\in\mathbb{R}$, initially set to zero.

For every $Z\in \JTpower{C}$, in turn, we do the following:
\begin{equation}
h(Z\cap W)\leftarrow h(Z\cap W) + \prod_{i=1}^k\Upsilon_i(Z\cap D_i)
\end{equation}
Note that after we have performed the above for every $Z\in \JTpower{C}$, the function $h$ is equal to the potential  $\JTmarg{\left(\prod_{i=1}^k \Upsilon_i\right)}{W}$. We then output the potential $h$.
\end{algorithm}

\noindent\hfil\rule{1\textwidth}{.4pt}\hfil

If algorithm \ref{JTalg1} is used for performing operation \ref{JTop1} then the computation of each message $\JTmess{ C}{ E}$ takes a time of $\Theta((\JTdegree{ C}+|\JTfactor{ C}|)\JTexp{| C|})$ and requires only constant auxiliary space. Hence, the space complexity of the entire message passing algorithm is the space required to store the factors and messages. Since each vertex $ C$ sends $\JTdegree{ C}$ messages, each vertex $ C$ contributes a time of  $\Theta(\JTdegree{ C}(\JTdegree{ C}+|\JTfactor{ C}|)\JTexp{ C})$ to the entire message passing algorithm.

\noindent\hfil\rule{1\textwidth}{.4pt}\hfil

\subsection{Hugin Propagation}
In this subsection we describe and analyse the complexity of Hugin propagation:\\
Hugin propagation stores the following potentials: For every vertex $ C\in\JTvertex{\JT}$ we have a potential $\Gamma_{ C}\in\JTpots{ C}$ initialised to be equal to $\JTone{C}\prod_{\Lambda\in\JTfactor{ C}}\Lambda$. For every edge $\{ C,  E\}\in\JTedge{\JT}$ we have a potential $\Psi_{\{ C,  E\}}\in\JTpots{ C\cap E}$ initialised equal to $\JTone{C\cap E}$. 
Hugin propagation follows the following algorithm:

\begin{algorithm} {\bf Hugin Propagtion:}\\
Given an ordered pair $( C,  E)$ of neighbouring vertices, once $ C$ has received messages from all vertices in $\JTneighbour{ C}\setminus\{ E\}$, it sends a message to $ E$ via the following algorithm:
\begin{enumerate}
\item Set $\Psi^{\operatorname{old}}_{\{ C,  E\}}\leftarrow\Psi_{\{ C,  E\}}$
\item Set $\Psi_{\{ C,  E\}}\leftarrow\JTmarg{\Gamma_{ C}}{ C\cap E}$
\item Set $\JTmess{ C}{ E}\leftarrow\Psi_{\{ C,  E\}}/\Psi^{\operatorname{old}}_{\{ C,  E\}}$
\item Set $\Gamma_{ E}\leftarrow\JTmess{ C}{ E}\Gamma_{ E}$
\end{enumerate}
\end{algorithm}

\noindent\hfil\rule{1\textwidth}{.4pt}\hfil
Note that the time required by a vertex $C$ to pass a message to a neighbour $ E$ is $\Theta(\JTexp{| C|}+\JTexp{| E|})$. Since each vertex $ C$ sends and receives a message to/from each of its neighbours, and since the potential $\Gamma_{ C}$ takes a time of $\Theta(|\JTfactor{ C}|\JTexp{| C|})$ to initialise, we have that $ C$ contributes a time of $\Theta((\JTdegree{ C}+|\JTfactor{ C}|)\JTexp{| C|})$ to the entire message passing algorithm. Note then that Hugin propagation is faster than Shafer-Shenoy propagation. The drawback, however, is that storing, for each vertex $ C$, the potential $\Psi_{ C}$ has a space requirement of $\Theta(\JTexp{| C|})$. This leads to a total space requirement of $\Theta(\sum_{ C\in\JTvertex{\JT}}\JTexp{| C|})$ which can be significantly more (and never less) than that of Shafer-Shenoy propagation.

\noindent\hfil\rule{1\textwidth}{.4pt}\hfil
Given a vertex $C\in\JTvertex{\JT}$, if the potential $\Gamma_{ C}$ is initialised by combining (via multiplication) the factors in $\JTfactor{C}$ on a binary basis (as is described in \cite {JTref11}) then the initialisation time of $\Gamma_{C}$ can be less than $\Theta\left(|\JTfactor{ C}|\JTexp{| C|}\right)$ so the time complexity of Hugin propagation can be reduced. However, each vertex still contributes a time of at least $\Omega\left(\JTdegree{C}\JTexp{| C|}\right)$ so if the degree of a vertex is greater than the number of associated factors then combining factors on a binary basis does not speed up this time by more than a constant factor. 
In addition, this faster version of Hugin propagation is still never faster than ARCH-2 by more than a logarithmic factor and when we have large vertices of high degree is still significantly slower than ARCH-2.

\noindent\hfil\rule{1\textwidth}{.4pt}\hfil
\section{ARCH-1}
\label{JTsection5}
In this section we describe the architecture ARCH-1 which has (up to a constant factor) the speed of Hugin propagation and the low space complexity of Shafer-Shenoy propagation. The reason for the low time/space complexity is that many marginals are computed simultaneously from a factored potential using a merger of the ideas behind Shafer-Shenoy and Hugin propagation: an algorithm similar to Algorithm \ref{JTalg1} and the division idea of the Hugin architecture. Like Shafer-Shenoy propagation we store only the messages. 
\newline
\newline
ARCH-1 selects a vertex $\JTroot$ as the root of $\JT$ and then (as is often in the description of Hugin and Shafer-Shenoy propagation) has two phases: the \textit{inward phase}, in which messages are passed up the tree to the root and the \textit{outward phase}, in which messages are passed down the tree from the root to the leaves. We first sketch an outline of ARCH-1 (which is also an outline of ARCH-2) before going into the details:
\begin{algorithm}
\label{JTalgorithm2second} {\bf Outline of ARCH-1/ARCH-2:}\\
The algorithm has two phases: First the \textit{inward phase} and then the \textit{outward phase}.
\begin{enumerate}
\item Inward phase: For every vertex $ C\in\JTvertex{\JT}\setminus\{\JTroot\}$, once $ C$ has received messages from all its children, it sends a message to its parent as follows:
\begin{enumerate} \item \label{JTnewfirst} The message $\JTmess{ C}{\JTpar{ C}}$ is computed as:
\begin{equation}
\label{JTfirst}\JTmess{ C}{\JTpar{ C}}\leftarrow\JTmarg{\left[\left(\prod_{\Lambda\in\JTfactor{ C}}\Lambda\right)\left(\boldsymbol{1}_{ C\cap\JTpar{ C}}\prod_{ H\in\JTchildren{ C}}\JTmess{ H}{ C}\right)\right]}{ C\cap\JTpar{ C}}
\end{equation}
and is sent from $ C$ to $\JTpar{ C}$.
\end{enumerate}
\item Outward phase: For every vertex $ C\in\JTvertex{\JT}$, once $ C$ has received messages from all its neighbours, it sends messages to all its children as follows:
\begin{enumerate}
\item \label{JTsecond} For every $ E\in\JTchildren{ C}$, simultaneously, the potential $M'_{ E}$ (in $\JTpots{ C\cap E}$) is computed as:
\begin{equation}
M'_{ E}\leftarrow\JTmarg{\left[\left(\prod_{\Lambda\in\JTfactor{ C}}\Lambda\right)\left(\prod_{ H\in\JTneighbour{ C}}\JTmess{ H}{ C}\right)\right]}{ C\cap E}
\end{equation}
\item \label{JTdivision} For every $ E\in\JTchildren{ C}$ the message $\JTmess{ C}{ E}$ is computed as:
\begin{equation}
\JTmess{ C}{ E}\leftarrow {M'_{ E}}/{\JTmess{ E}{ C}}
\end{equation}
and is sent to $ E$.
\end{enumerate}
\end{enumerate}
\end{algorithm}

\noindent\hfil\rule{1\textwidth}{.4pt}\hfil
 We now prove the correctness of Algorithm \ref{JTalgorithm2second}: i.e. that the messages are equal to those defined in Stage \ref{JTmessagepassing} of Algorithm \ref{JTA}.
 \newline
 \newline
Consider first the inward phase: Since Equation \ref{JTfirst} is the same as Equation \ref{JTmultimess} we have, by induction up $\JT$ from the leaves to the root, that $\JTmess{ C}{\JTpar{ C}}$ is correctly computed for every $ C\in\JTvertex{\JT}\setminus\{R\}$.
\newline
\newline
Consider next the outward phase: We prove, by induction on $ C$ down $\JT$ from the root to the leaves, that $\JTmess{ C}{ E}$ is correctly computed for all $ E\in\JTchildren{ C}$. By the inductive hypothesis and the result above that $\JTmess{ H}{{ C}}$ is correctly computed for every $ H\in\JTchildren{ C}$ we have that $\JTmess{ H}{{ C}}$ is correctly computed for every $ H\in\JTneighbour{ C}$. Hence for all $ E\in\JTchildren{ C}$ and all $Y\in\JTpower{ E}$ we have:
\begin{align}
M'_{ E}(Y)&=\JTmarg{\left[\left(\prod_{\Lambda\in\JTfactor{ C}}\Lambda\right)\left(\prod_{ H\in\JTneighbour{ C}}\JTmess{ H}{ C}\right)\right]}{ C\cap E}(Y)\\
&=\JTmarg{\left[\left(\prod_{\Lambda\in\JTfactor{ C}}\Lambda\right)\JTone{ C\cap E}\left(\prod_{ H\in\JTneighbour{ C}}\JTmess{ H}{ C}\right)\right]}{ C\cap E}(Y)\\
&=\sum_{Z\in\JTpower{ C}:Z\cap C\cap E=Y}\left[\left(\prod_{\Lambda\in\JTfactor{ C}}\Lambda\right)\JTone{ C\cap E}\left(\prod_{ H\in\JTneighbour{ C}}\JTmess{ H}{ C}\right)\right](Z)\\
&=\sum_{Z\in\JTpower{ C}:Z\cap C\cap E=Y}\JTmess{ E}{ C}(Z\cap C\cap E)\left[\left(\prod_{\Lambda\in\JTfactor{ C}}\Lambda\right)\JTone{ C\cap E}\left(\prod_{ H\in\JTneighbour{ C}\setminus\{ E\}}\JTmess{ H}{ C}\right)\right](Z)\\
&=\sum_{Z\in\JTpower{ C}:Z\cap C\cap E=Y}\JTmess{ E}{ C}(Y)\left[\left(\prod_{\Lambda\in\JTfactor{ C}}\Lambda\right)\JTone{ C\cap E}\left(\prod_{ H\in\JTneighbour{ C}\setminus\{ E\}}\JTmess{ H}{ C}\right)\right](Z)\\
&=\JTmess{ E}{ C}(Y)\sum_{Z\in\JTpower{ C}:Z\cap C\cap E=Y}\left[\left(\prod_{\Lambda\in\JTfactor{ C}}\Lambda\right)\JTone{ C\cap E}\left(\prod_{ H\in\JTneighbour{ C}\setminus\{ E\}}\JTmess{ H}{ C}\right)\right](Z)\\
&=\JTmess{ E}{ C}(Y)\JTmarg{\left[\left(\prod_{\Lambda\in\JTfactor{ C}}\Lambda\right)\left(\prod_{ H\in\JTneighbour{ C}\setminus\{ E\}}\JTmess{ H}{ C}\right)\right]}{ C\cap E}(Y)\\
&=\JTmess{ E}{ C}(Y)\JTmess{ C}{ E}(Y)
\end{align}
and hence $[M'_{ E}/\JTmess{ E}{ C}](Y)=\JTmess{ C}{ E}(Y)$ so $M'_{ E}/\JTmess{ E}{ C}=\JTmess{ C}{ E}$ which proves that the inductive hypothesis holds for $ C$

\noindent\hfil\rule{1\textwidth}{.4pt}\hfil

Note that Step \ref{JTnewfirst} and Step \ref{JTsecond} of Algorithm \ref{JTalgorithm2second} can be solved by instances of the following operation (where $\{D_1, D_2, ..., D_k\}:=\{\JTunderlying{\Lambda}:\Lambda\in\JTfactor{ C}\}\cup\{ C\cap H: H\in\JTneighbour{ C}\}$):

\begin{operation}
\label{JToper1}
We have, as input, a set $C$, and subsets $\{D_1, D_2, ... , D_k\}\subseteq\JTpower{C}$ with $\bigcup_{i=1}^kD_i=C$. For every $i\in\JTnatup{k}$ we have, as input, a potential $\Upsilon_i\in\JTpots{D_i}$.
\newline
Define $\Gamma:=\prod_{i=1}^k\Upsilon_i$ and for every $i\in\JTnatup{k}$ define $\Psi_i:=\JTmarg{\Gamma}{D_i}$.
\newline
We must compute $\Psi_i$ for every $i\in\JTnatup{k}$.
\end{operation}

\noindent\hfil\rule{1\textwidth}{.4pt}\hfil

ARCH-1 computes operation \ref{JToper1} via the following algorithm, which can be implemented in a time of $\Theta(k\JTexp{|C|})$ using only constant auxiliary space:

\begin{algorithm}
\label{JTalgorithm1}
For every $i\in\JTnatup{k}$ and every $Y\in \JTpower{D_i}$ we maintain a variable $h_i(Y)\in\mathbb{R}$ initially set equal to zero.
\newline
For every $Z\in \JTpower{C}$, in turn, we do the following:
\begin{enumerate}
\item Set $\alpha\leftarrow\prod_{i=1}^k\Upsilon_i(Z\cap D_i)$
\item For all $i\in\JTnatup{k}$ set $h_i(Z\cap D_i)\leftarrow h_i(Z\cap D_i)+\alpha$.
\end{enumerate}
Note that after we have performed the above for every $Z\in \JTpower{C}$, the function $h_i$ is a potential in $\JTpots{D_i}$. We then output, for every $i\in\JTnatup{k}$, $\Psi_i\leftarrow h_i$.
\end{algorithm}

\noindent\hfil\rule{1\textwidth}{.4pt}\hfil

The correctness of Algorithm \ref{JTalgorithm1} is seen immediately by noting that at the end of the algorithm we have, for all $i\in\JTnatup{k}$ and $Y\in\JTpower{D_i}$:
\begin{align}
h_i(Y)&=\sum_{[Z\in\JTpower{C}:Z\cup D_i=Y]}\prod_{i=1}^k\Upsilon_i(Z\cap D_i)\\
&=\sum_{Z\in\JTpower{C}:Z\cup D_i=Y}\left[\prod_{i=1}^k\Upsilon_i\right](Z)\\
&=\sum_{Z\in\JTpower{C}:Z\cup D_i=Y}\Gamma(Z)\\
&=\JTmarg{\Gamma}{D_i}(Y)\\
&=\Psi_i(Y)
\end{align}
where $\Gamma$ is as in the statement of operation \ref{JToper1}.

\noindent\hfil\rule{1\textwidth}{.4pt}\hfil

Note that, for every vertex $ C$, Operation \ref{JToper1} is called twice (once during the inward phase and once during the outward phase), each time taking a time, under Algorithm \ref{JTalgorithm1}, of $\Theta((\JTdegree{ C}+|\JTfactor{ C}|)\JTexp{| C|})$. Each vertex $ C$ hence contributes a time of $\Theta((\JTdegree{ C}+|\JTfactor{ C}|)\JTexp{| C|})$ to the time complexity of the whole message passing algorithm. ARCH-1 hence has the same time complexity as Hugin propagation. Like Shafer-Shenoy propagation, the space required by ARCH-1 is only 
that required to store the factors and messages.

\noindent\hfil\rule{1\textwidth}{.4pt}\hfil
In Section \ref{JTalgo1imp} we show how, by caching various quantities, we can, whilst keeping the same space requirements, speed up Algorithm \ref{JTalgorithm1} to take a time of only $\Theta\left(\sum_{i=1}^{|C|}|\{j:y_i\in D_j\}|\JTexp{i}\right)$ where $\{y_j:j\in\JTnatup{|C|}\}:=C$. In order be free to choose the ordering $(y_1, y_2, ..., y_{|C|})$ of $C$ that minimises this time we require an additional time of $\Theta\left(\sum_{i=1}^k|D_i|\JTexp{|D_i|}\right)$. However, even this faster implementation of ARCH-1 may still be significantly slower than ARCH-2 and will never be faster by more than a logarithmic factor.

\noindent\hfil\rule{1\textwidth}{.4pt}\hfil
\section{ARCH-2}
\label{JTsection6}
We now describe the architecture ARCH-2. The time and space complexities of ARCH-2 are always (up to a factor that is linear in the width of $\JT$) at least as good as those of ARCH-1. In cases in which we have large vertices of high degree ARCH-2 is significantly faster than ARCH-1/Hugin.\newline\newline Specifically, each vertex $ C$ contributes a time of only $\Theta(| C|\exp{| C|})$ to ARCH-2 and, in addition to storing the factors and messages, ARCH-2 requires a space of only \newline$\mathcal{O}\left(\operatorname{max}_{ C\in\JTvertex{\JT}}| C|\left(\left(\sum_{ H\in\JTneighbour{ C}}\JTexp{| H\cap C|}\right)+\left(\sum_{\Lambda\in\JTfactor{ C}}\JTexp{|\JTunderlying{ C}}\right)\right)\right)$.
\newline
\newline
ARCH-2 proceeds similarly to ARCH-1, using Operation \ref{JToper1} to do steps \ref{JTnewfirst} and \ref{JTsecond} of Algorithm \ref{JTalgorithm2second}. The only difference between ARCH-2 and ARCH-1 is how Operation \ref{JToper1} is computed. The algorithm for performing Operation \ref{JToper1} is based upon the concepts of the \textit{p-dual} and the \textit{m-dual} of a potential. We first give a definition of the duals and the required theory surrounding them.

\noindent\hfil\rule{1\textwidth}{.4pt}\hfil

\subsection{The p-Dual }
\label{pdualsection}

In this subsection we introduce the p-dual and the required theory surrounding it.
We first define the p-dual of a potential:

\begin{definition}
\label{JTdef1} {\bf The p-dual:}

Given a set $X$ and a potential $\Phi\in\JTpots{X}$, the \textit{p-dual}, $\JTdual{\Phi}$, of $\Phi$ is the potential in $\JTpots{X}$ that satisfies, for every $Y\in\JTpower{X}$:
\begin{equation}
\JTdual{\Phi}(Y)=\prod_{Z\in\JTpower{Y}}\JTperm{|Z|}{\Phi(Z)}
\end{equation}
\end{definition}

\noindent\hfil\rule{1\textwidth}{.4pt}\hfil

The next theorem will assist us in the the recovery of potential from its p-dual 

\begin{theorem}
\label{JTtheorem2}
Suppose we have a set $X$, an element $x\in X$ and a potential $\Phi\in\JTpots{X}$. Define $\Phi_{-}$ and $\Phi_{+}$ to be the potentials in $\JTpots{X\setminus\{x\}}$ that satisfy, for every $Z\in\JTpower{X\setminus\{x\}}$, $\Phi_{-}(Z):=\Phi(Z)$ and $\Phi_{+}(Z):=\Phi(Z\cup\{x\})$. For every $Y\in\JTpower{X\setminus\{x\}}$ we have the following:
\begin{enumerate}
\item $\JTdual{\Phi_{-}}(Y)=\JTdual{\Phi}(Y)$
\item $\JTdual{\Phi_{+}}(Y)={\JTdual{\Phi}(Y)}/{\JTdual{\Phi}(Y\cup\{x\})}$
\end{enumerate}
\end{theorem}

\begin{proof}
\begin{enumerate}
\item
\begin{align}
\JTdual{\Phi_{-}}(Y)&=\prod_{Z\in\JTpower{Y}}\JTperm{|Z|}{\Phi_{-}(Z)}\\
&=\prod_{Z\in\JTpower{Y}}\JTperm{|Z|}{\Phi(Z)}\\
&=\JTdual{\Phi}(Y)
\end{align}
\item
\begin{align}
\JTdual{\Phi_{+}}(Y)&=\prod_{Z\in\JTpower{Y}}\JTperm{|Z|}{\Phi_{+}(Z)}\\
&=\prod_{Z\in\JTpower{Y}}\JTperm{|Z|}{\Phi(Z\cup\{x\})}\\
&=\prod_{U\in\JTpower{Y\cup\{x\}:x\in U}}\JTperm{|U|-1}{\Phi(U)}\\
&=\left(\prod_{U\in\JTpower{Y\cup\{x\}:x\in U}}\JTperm{|U|}{\Phi(U)}\right)^{-1}\\
&=\left(\prod_{U\in\JTpower{Y\cup\{x\}}\setminus\JTpower{Y}}\JTperm{|U|}{\Phi(U)}\right)^{-1}\\
&=\left(\prod_{U\in\JTpower{Y}}\JTperm{|U|}{\Phi(U)}\right)\left(\prod_{U\in\JTpower{Y\cup\{x\}}}\JTperm{|U|}{\Phi(U)}\right)^{-1}\\
&=\frac{\JTdual{\Phi}(Y)}{\JTdual{\Phi}(Y\cup\{x\})}
\end{align}
\end{enumerate}
\end{proof}

\noindent\hfil\rule{1\textwidth}{.4pt}\hfil

From Theorem \ref{JTtheorem2} we get the following theorem, which will aid us in the construction of a p-dual.

\begin{theorem}
\label{JTtheorem24}
Suppose we have a set $X$, an element $x\in X$ and a potential $\Phi\in\JTpots{X}$. Define $\Phi_{-}$ and $\Phi_{+}$ to be the potentials in $\JTpots{X\setminus\{x\}}$ that satisfy, for every $Z\in\JTpower{X\setminus\{x\}}$, $\Phi_{-}(Z):=\Phi(Z)$ and $\Phi_{+}(Z):=\Phi(Z\cup\{x\})$. For every $Y\in\JTpower{X\setminus\{x\}}$ we have the following:
\begin{enumerate}
\item $\JTdual{\Phi}(Y)=\JTdual{\Phi_{-}}(Y)$
\item $\JTdual{\Phi}(Y\cup\{x\})={\JTdual{\Phi_{-}}(Y)}/{\JTdual{\Phi_{+}}(Y)}$
\end{enumerate}
\end{theorem}

\begin{proof}
The result comes from solving the equations of Theorem \ref{JTtheorem2} for $\JTdual{\Phi}(Y)$ and $\JTdual{\Phi}(Y\cup\{x\})$
\end{proof}

\noindent\hfil\rule{1\textwidth}{.4pt}\hfil

We next show that the p-dual of a product of potentials with the same domain is the product of the p-duals of the potentials:

\begin{lemma}
\label{JTlem4}
Given a set $X$ and potentials $\Phi,\Phi'\in\JTpots{X}$, we have $\JTdual{(\Phi'\Phi)}=(\JTdual{\Phi'})(\JTdual{\Phi})$
\end{lemma}

\begin{proof}
For any $Y\in\JTpower{X}$ we have:
\begin{align}
\JTdual{[\Phi'\Phi]}(Y)&=\prod_{Z\in\JTpower{Y}}\JTperm{|Z|}{[\Phi'\Phi](Z)}\\
&=\prod_{Z\in\JTpower{Y}}\JTperm{|Z|}{[\Phi'(Z)\Phi(Z)]}\\
&=\prod_{Z\in\JTpower{Y}}\JTperm{|Z|}{\Phi'(Z)}\JTperm{|Z|}{\Phi(Z)}\\
&=\left(\prod_{Z\in\JTpower{Y}}\JTperm{|Z|}{\Phi'(Z)}\right)\left(\prod_{Z\in\JTpower{Y}}\JTperm{|Z|}{\Phi'(Z)}\right)\\
&=[\JTdual{\Phi'}(Y)][\JTdual{\Phi}(Y)]
\end{align}
\end{proof}

\noindent\hfil\rule{1\textwidth}{.4pt}\hfil

With the aid of the following lemma we will show how to compute the p-dual of the product of small potentials:

\begin{lemma}
\label{JTlem3}
Given a set $X$, a set $Y\subseteq X$, and a potential $\Phi\in\JTpots{Y}$, let $\Phi'$ be the potential in $\JTpots{X}$ that satisfies, for all $Z\in\JTpower{X}$, $\Phi'(Z):=\Phi(Z\cap Y)$. Then given $U\in\JTpower{X}$ we have:
\begin{enumerate}
\item If $U\subseteq Y$ then $\JTdual{\Phi'}(U)=\JTdual{\Phi}(U)$
\item If $U\nsubseteq Y$ then $\JTdual{\Phi'}(U)=1$
\end{enumerate}
\end{lemma}

\begin{proof}
\begin{enumerate}
\item 
\begin{align}
\JTdual{\Phi'}(U)&=\prod_{Z\in\JTpower{U}}\JTperm{|Z|}{\Phi'(Z)}\\
&=\prod_{Z\in\JTpower{U}}\JTperm{|Z|}{\Phi(Z\cap Y)}\\
&=\prod_{Z\in\JTpower{U}}\JTperm{|Z|}{\Phi(Z)}\\
&=\JTdual{\Phi}(U)
\end{align}
\item
We have $U\setminus Y\neq\emptyset$ so choose some $v\in U\setminus Y$. We then have:
\begin{align}
 \JTdual{\Phi'}(U)&=\prod_{Z\in\JTpower{U}}\JTperm{|Z|}{\Phi'(Z)}\\
&=\prod_{W\in\JTpower{U\setminus\{v\}}}\JTperm{|W|}{\Phi'(W)}\JTperm{|W\cup\{v\}|}{\Phi'(W\cup\{v\})}\\
&=\prod_{W\in\JTpower{U\setminus\{v\}}}\JTperm{|W|}{\Phi(W\cap Y)}\JTperm{|W\cup\{v\}|}{\Phi((W\cup\{v\})\cap Y)}\\
&=\prod_{W\in\JTpower{U\setminus\{v\}}}\JTperm{|W|}{\Phi(W\cap Y)}\JTperm{|W\cup\{v\}|}{\Phi(W\cap Y)}\\
&=\prod_{W\in\JTpower{U\setminus\{v\}}}\JTperm{|W|}{\Phi(W\cap Y)}\JTperm{|W|+1}{\Phi(W\cap Y)}\\
&=\prod_{W\in\JTpower{U\setminus\{v\}}}\JTperm{|W|}{\Phi(W\cap Y)}{\Phi(W\cap Y)}^{-(-1)^{|W|}}\\
&=\prod_{W\in\JTpower{U\setminus\{v\}}}{\Phi(W\cap Y)}^{(-1)^{|W|}-(-1)^{|W|}}\\
&=\prod_{W\in\JTpower{U\setminus\{v\}}}{\Phi(W\cap Y)}^{0}\\
&=1
 \end{align}
 \end{enumerate}
\end{proof}

\noindent\hfil\rule{1\textwidth}{.4pt}\hfil

\begin{theorem}
\label{JTtheorem3}
Suppose we have a set $X$, subsets $\{Y_i:i\in\mathbb{N}_k\}\subseteq\JTpower{X}$ such that $\bigcup\{Y_i:i\in\mathbb{N}_k\}=X$ and potentials $\{\Phi_i:i\in\mathbb{N}_k\}$ such that $\Phi_i\in\JTpots{Y_i}$. Then for every $U\in\JTpower{X}$ we have:
\begin{equation}
\JTdual{\left[\prod_{i=1}^k\Phi_i\right]}(U)=\prod_{i:U\subseteq Y_i}\JTdual{\Phi_i}(U)
\end{equation}
\end{theorem}

\begin{proof}
For $i\in\mathbb{N}_k$ let $\Phi'_i$ be the potential in $\JTpots{X}$ that satisfies, for all $Z\in\JTpower{X}$, $\Phi'(Z):=\Phi(Z\cap Y)$. Then we have:
\begin{align}
\JTdual{\left[\prod_{i=1}^k\Phi_i\right]}(U)&=\JTdual{\left[\prod_{i=1}^k\Phi'_i\right]}(U)\\
\label{JTeq10}&=\prod_{i=1}^k\JTdual{\Phi'_i}(U)\\
&=\left(\prod_{i:U\subseteq Y_i}\JTdual{\Phi'_i}(U)\right)\left(\prod_{i:U\not\subseteq Y_i}\JTdual{\Phi'_i}(U)\right)\\
\label{JTeq11}&=\left(\prod_{i:U\subseteq Y_i}\JTdual{\Phi_i}(U)\right)\left(\prod_{i:U\not\subseteq Y_i}1\right)\\
&=\prod_{i:U\subseteq Y_i}\JTdual{\Phi_i}(U)
\end{align}
Where equation \ref{JTeq10} comes from lemma \ref{JTlem4} and equation \ref{JTeq11} comes from lemma \ref{JTlem3}.
\end{proof}

\noindent\hfil\rule{1\textwidth}{.4pt}\hfil

\subsection{The m-Dual}
\label{mdualsection}
In this subsection we introduce the m-dual and the required theory surrounding it. The m-dual was defined in \cite{JTref7} under the name of ``inclusion-exclusion format" but was used in a very different way. We first define the m-dual of a potential:

\begin{definition}
\label{JTdef1} {\bf The m-dual:}

Given a set $X$ and a potential $\Phi\in\JTpots{X}$, the \textit{m-dual}, $\JTmdual{\Phi}$, of $\Phi$ is the potential in $\JTpots{X}$ that satisfies, for every $Y\in\JTpower{X}$:
\begin{equation}
\JTmdual{\Phi}(Y)=\sum_{Z\in\JTpower{X}:Y\subseteq Z}\Phi(Z)
\end{equation}
\end{definition}

\noindent\hfil\rule{1\textwidth}{.4pt}\hfil

The following theorem will be useful in the construction of an m-dual:

\begin{theorem}
\label{JTtheorem20}
Suppose we have a set $X$, an element $x\in X$ and a potential $\Phi\in\JTpots{X}$. Define $\Phi_{-}$ and $\Phi_{+}$ to be the potentials in $\JTpots{X\setminus\{x\}}$ that satisfy, for every $Z\in\JTpower{X\setminus\{x\}}$, $\Phi_{-}(Z):=\Phi(Z)$ and $\Phi_{+}(Z):=\Phi(Z\cup\{x\})$. For every $Y\in\JTpower{X\setminus\{x\}}$ we have the following:
\begin{enumerate}
\item $\JTmdual{\Phi}(Y)=\JTmdual{\Phi_{-}}(Y)+\JTmdual{\Phi_{+}}(Y)$
\item $\JTmdual{\Phi}(Y\cup\{x\})=\JTmdual{\Phi_{+}}(Y)$
\end{enumerate}
\end{theorem}

\begin{proof}
\begin{enumerate}
\item
\begin{align}
\JTmdual{\Phi}(Y)&=\sum_{Z\in\JTpower{X}:Y\subseteq Z}\Phi(Z)\\
&=\left(\sum_{Z\in\JTpower{X}:Y\subseteq Z, x\notin Z}\Phi(Z)\right)+\left(\sum_{Z\in\JTpower{X}:Y\subseteq Z, x\in Z}\Phi(Z)\right)\\
&=\left(\sum_{Z\in\JTpower{X\setminus\{x\}}:Y\subseteq Z}\Phi(Z)\right)+\left(\sum_{Z\in\JTpower{X}:Y\subseteq Z, x\in Z}\Phi(Z)\right)\\
\label{JTequation100}&=\left(\sum_{Z\in\JTpower{X\setminus\{x\}}:Y\subseteq Z}\Phi(Z)\right)+\left(\sum_{U\in\JTpower{X\setminus\{x\}}:Y\subseteq U}\Phi(U\cup\{x\})\right)\\
&=\left(\sum_{Z\in\JTpower{X\setminus\{x\}}:Y\subseteq Z}\Phi_{-}(Z)\right)+\left(\sum_{U\in\JTpower{X\setminus\{x\}}:Y\subseteq U}\Phi_{+}(U)\right)\\
&=\JTmdual{\Phi_{-}}(Y)+\JTmdual{\Phi_{+}}(Y)
\end{align}
where Equation \ref{JTequation100} comes from setting $U:=Z\setminus\{x\}$ in the right-hand sum
\item
\begin{align}
\JTmdual{\Phi}(Y\cup\{x\})&=\sum_{Z\in\JTpower{X}:Y\cup\{x\}\subseteq Z}\Phi(Z)\\
\label{JTequation101}&=\sum_{U\in\JTpower{X\setminus\{x\}}:Y\subseteq U}\Phi(U\cup\{x\})\\
&=\sum_{U\in\JTpower{X\setminus\{x\}}:Y\subseteq U}\Phi_{+}(U)\\
&=\JTmdual{\Phi_{+}}(Y)
\end{align}
where Equation \ref{JTequation101} comes from setting $U:=Z\setminus\{x\}$.
\end{enumerate}
\end{proof}

\noindent\hfil\rule{1\textwidth}{.4pt}\hfil

From Theorem \ref{JTtheorem20} we get the following theorem, which will be useful in converting an m-dual back to the original potential:

\begin{theorem}
\label{JTtheorem21}
Suppose we have a set $X$, an element $x\in X$ and a potential $\Phi\in\JTpots{X}$. Define $\Phi_{-}$ and $\Phi_{+}$ to be the potentials in $\JTpots{X\setminus\{x\}}$ that satisfy, for every $Z\in\JTpower{X\setminus\{x\}}$, $\Phi_{-}(Z):=\Phi(Z)$ and $\Phi_{+}(Z):=\Phi(Z\cup\{x\})$. For every $Y\in\JTpower{X\setminus\{x\}}$ we have the following:
\begin{enumerate}
\item $\JTmdual{\Phi_{-}}(Y)=\JTmdual{\Phi}(Y)-\JTmdual{\Phi}(Y\cup\{x\})$
\item $\JTmdual{\Phi_{+}}(Y)=\JTmdual{\Phi}(Y\cup\{x\})$
\end{enumerate}
\end{theorem}

\begin{proof}
The result comes from solving the equations of Theorem \ref{JTtheorem20} for $\JTmdual{\Phi_{-}}(Y)$ and $\JTmdual{\Phi_{+}}(Y)$
\end{proof}

\noindent\hfil\rule{1\textwidth}{.4pt}\hfil

We next show how marginals are computed when working with m-duals:

\begin{theorem}
\label{JTtheorem22}
Suppose we have sets $X$ and $Y$ with $Y\subseteq X$ and a potential $\Phi\in\JTpots{X}$. Then for every $Z\in\JTpower{Y}$ we have:
\begin{equation}
\JTmdual{\left[\JTmarg{\Phi}{Y}\right]}(Z)=\JTmdual{\Phi}(Z)
\end{equation}
\end{theorem}

\begin{proof}
\begin{align}
\JTmdual{\left[\JTmarg{\Phi}{Y}\right]}(Z)&=\sum_{U\in\JTpower{Y}:Z\subseteq U}\JTmarg{\Phi}{Y}(U)\\
&=\sum_{[U\in\JTpower{Y}:Z\subseteq U]}\sum_{[W\in\JTpower{X}:W\cap Y= U]}\Phi(W)
\end{align}
Note that if we have $U, U' \in \JTpower{Y}$ with $U\neq U'$ and we have $W, W'\in\JTpower{X}$ with $W\cap Y=U$ and $W'\cap Y=U'$ then $W\cap Y\neq W'\cap Y$ so $W\neq W'$. Hence, each $W$ is the (double) is counted only once.
\newline
Suppose we have $W\in\JTpower{X}$ with $Z\subseteq W$. Then if $U:= W\cap Y$ then since $Z\subseteq Y$ and $Z\subseteq W$ we have $Z\subseteq U$ so $W$ is included in the (double) sum.
\newline
Now suppose $W$ is included in the (double) sum. Then there exists a $U\in \JTpower{Y}$ with $Z\subseteq U$ such that $W \cap Y =U$. Hence $Z\subseteq W \cap Y$ so $Z\subseteq W$.
\newline
Hence, for each $W\in \JTpower{X}$, $W$ is contained in the (double) sum if and only if $Z \subseteq W$ and so since, by above, each such $W$ is counted only once in the (double) sum we have:
\begin{align}
\JTmdual{\left[\JTmarg{\Phi}{Y}\right]}(Z)&=\sum_{[U\in\JTpower{Y}:Z\subseteq U]}\sum_{[W\in\JTpower{X}:W\cap Y= U]}\Phi(W)\\
&=\sum_{W\in\JTpower{X}:Z\subseteq W}\Phi(W)\\
&=\JTmdual{\Phi}(Z)
\end{align}
\end{proof}

\noindent\hfil\rule{1\textwidth}{.4pt}\hfil

\subsection{Functions for Manipulating Potentials}
\label{JTfunctions}

We now describe the functions used by ARCH-2. The functions are {\bf transform1} which transforms a potential into its p-dual, {\bf product} which computes the product of potentials when working with p-duals, {\bf transform2} which transforms the p-dual of a potential (in a sparse format) into the m-dual of the potential (in a sparse format), {\bf marginalise} which computes marginals of a potential while working with m-duals, and {\bf transform3} which transforms the m-dual of a potential back to the original potential.
\newline\newline
The functions {\bf transform1}, {\bf transform2} and {\bf transform3} all rest on the observation that, given a potential $\Phi$ with $\JTunderlying{\Phi}=\emptyset$, we have $\JTmdual{\Phi}=\JTdual{\Phi}=\Phi$.
\newline\newline In the description of the functions {\bf transform1}, {\bf transform2} and {\bf transform3}, $\Phi_{-}$ and $\Phi_{+}$ are defined from $\Phi$ and $x$ as in the statements of theorems \ref{JTtheorem2}, \ref{JTtheorem24}, \ref{JTtheorem20} and \ref{JTtheorem21}: i.e. $\Phi_{-}$ and $\Phi_{+}$ are the potentials in $\JTpots{\JTunderlying{\Phi}\setminus\{x\}}$ that satisfy, for every $Z\in\JTpower{\JTunderlying{\Phi}\setminus\{x\}}$, $\Phi_{-}(Z):=\Phi(Z)$ and $\Phi_{+}(Z):=\Phi(Z\cup\{x\})$
\newline\newline For a detailed description of how to implement these functions so they have the stated time and space complexities see section \ref{JTFuncImpSection} (which is based on notation and algorithms given earlier in section \ref{JTsection8})

\noindent\hfil\rule{1\textwidth}{.4pt}\hfil
We first describe the recursive function {\bf transform1}:
\begin{itemize}
\item The function takes, as input, a potential $\Phi$
\item The function outputs the p-dual, $\JTdual{\Phi}$, of $\Phi$. 
\item The algorithm can be implemented to take a time of $\Theta\left(|\JTunderlying{\Phi}|\JTexp{|\JTunderlying{\Phi}|}\right)$ and to require  $\Theta\left(\JTexp{|\JTunderlying{\Phi}|}\right)$ auxiliary space. This is proved immediately by induction over $|\JTunderlying{\Phi}|$.
\item The correctness of the algorithm comes directly from Theorem \ref{JTtheorem24}, using induction over $|\JTunderlying{\Phi}|$.
\end{itemize}

\begin{algorithm} \label{JTf1}
{\bf transform1}$(\Phi)$:
\newline
If $\JTunderlying{\Phi}=\emptyset$ then return $\Phi$. Else, perform the following:
\begin{enumerate}
\item \label{JTf1step1} Choose $x\in\JTunderlying{\Phi}$.
\item \label{JTf1step2} For each $Z\in\JTpower{\JTunderlying{\Phi}\setminus\{x\}}$ set $\Phi_{-}(Z)\leftarrow\Phi(Z)$ and $\Phi_{+}(Z)\leftarrow\Phi(Z\cup\{x\})$. Note that $\Phi_{-}$ and $\Phi_{+}$ are now potentials in $\JTpots{\JTunderlying{\Phi}\setminus\{x\}}$.
\item \label{JTf1step3} Set $\JTdual{\Phi_{-}}\leftarrow${\bf transform1}$(\Phi_{-})$ and $\JTdual{\Phi_{+}}\leftarrow${\bf transform1}$(\Phi_{+})$.
\item \label{JTf1step4} For each $Y\in\JTpower{\JTunderlying{\Phi}\setminus\{x\}}$ set $\JTdual{\Phi}(Y)\leftarrow\JTdual{\Phi_{-}}(Y)$ and $\JTdual{\Phi}(Y\cup\{x\})\leftarrow{\JTdual{\Phi_{-}}(Y)}/{\JTdual{\Phi_{+}}(Y)}$
\item \label{JTf1step5} Return $\JTdual{\Phi}$.
\end{enumerate}
\end{algorithm}

\noindent\hfil\rule{1\textwidth}{.4pt}\hfil

We now describe the function {\bf product}:
\begin{itemize}
\item The function takes, as input, a set $\{\JTdual{\Upsilon_i}: i\in\JTnatup{k}\}$ of p-duals of potentials $\Upsilon_i$.
\item The function outputs the sparse format  $\JTsp{[\JTdual{\Gamma}]}{\zeta}$ where $\zeta=\bigcup_{i=1}^k\{\JTpower{\JTunderlying{\Upsilon_i}}\}$ and $\Gamma=\prod_{i=1}^k\Upsilon_i$. It is the case that for all $Z\in\JTunderlying{\Gamma}$ with $Z\notin\zeta$ we have $\JTdual{\Gamma}(Z)=1$. 
\item The algorithm can be implemented to take a time of $\mathcal{O}\left(\JTexp{|\bigcup_{i=1}^k\JTunderlying{\Upsilon_i}|}+\sum_{i=1}^k\JTexp{|\JTunderlying{\Upsilon_i}|}\right)$ and to require $\Theta(|\zeta|)$ auxiliary space.
\item The correctness of the algorithm comes directly from Theorem \ref{JTtheorem3}.
\end{itemize}

\begin{algorithm} \label{JTf2}
{\bf product}$(\{\JTdual{\Upsilon_i}: i\in\JTnatup{k}\})$:
\begin{enumerate}
\item \label{JTf2step1} Let $\zeta\leftarrow\bigcup_{i=1}^k\{\JTpower{\JTunderlying{\JTdual{\Upsilon_i}}}\}$
\item \label{JTf2step2} For each $Z\in\zeta$ set $\JTdual{\Gamma}(Z)\leftarrow\prod_{i\in\JTnatup{k}:Z\subseteq\JTunderlying{\Upsilon_i}}\JTdual{\Upsilon_i}(Z)$.
\item Return $\JTsp{[\JTdual{\Gamma}]}{\zeta}$
\end{enumerate}
\end{algorithm}

\noindent\hfil\rule{1\textwidth}{.4pt}\hfil

We now describe the recursive function {\bf transform2}:
\begin{itemize}
\item The function takes, as input, a sparse format, $\JTsp{[\JTdual{\Phi}]}{\zeta}$, of the p-dual of a potential $\Phi$ where $\zeta$ is a straddle-set such that, for all $Z\in\JTunderlying{\Phi}$ with $Z\notin\zeta$, we have $\JTdual{\Phi}(Z)=1$.
\item The function outputs the sparse format, $\JTsp{[\JTmdual{\Phi}]}{\zeta}$, of the m-dual of $\Phi$. 
\item The algorithm can be implemented to take a time of $\mathcal{O}(|\JTunderlying{\Phi}|\JTexp{|\JTunderlying{\Phi|}})$ and to require a space of only $\mathcal{O}(|\JTunderlying{\Phi}||\zeta|)$. This is proved immediately by induction over $|\JTunderlying{\Phi}|$, noting that for $x\in\JTunderlying{\Phi}$ we have that $|\{U\in\zeta: x\notin U\}|\leq|\zeta|$.
\item The correctness of the algorithm comes directly from theorems \ref{JTtheorem2} and \ref{JTtheorem20}, using induction over $|\JTunderlying{\Phi}|$ 
\end{itemize}

\begin{algorithm} \label{JTf3}
{\bf transform2}$(\JTsp{[\JTdual{\Phi}]}{\zeta})$:
\newline
If $\JTunderlying{\JTdual{\Phi}}=\emptyset$ then return $\Phi$. Else, perform the following:
\begin{enumerate}
\item \label{JTf3step1} Choose $x\in\JTunderlying{\Phi}$
\item \label{JTf3step2} Set $\vartheta\leftarrow\{U\in\zeta: x\notin U\}$
\item \label{JTf3step3} For each $Y\in\vartheta$ set $\JTdual{\Phi_{-}}(Y)\leftarrow\JTdual{\Phi}(Y)$
\item \label{JTf3step4} For each $Y\in\vartheta$ do the following:
\newline
If $Y\cup\{x\}\in\zeta$ then set $\JTdual{\Phi_{+}}(Y)\leftarrow\JTdual{\Phi}(Y)/\JTdual{\Phi}(Y\cup\{x\})$. Else set $\JTdual{\Phi_{+}}(Y)\leftarrow\JTdual{\Phi}(Y)$
\item \label{JTf3step5} Set $\JTsp{[\JTmdual{\Phi_{-}}]}{\vartheta}\leftarrow$ {\bf transform2}$(\JTsp{[\JTdual{\Phi_{-}}]}{\vartheta})$ and $\JTsp{[\JTmdual{\Phi_{+}}]}{\vartheta}\leftarrow$ {\bf transform2}$(\JTsp{[\JTdual{\Phi_{+}}]}{\vartheta})$
\item \label{JTf3step6} For each $Y\in\vartheta$ set $\JTmdual{\Phi}(Y)\leftarrow\JTmdual{\Phi_{-}}(Y)+\JTmdual{\Phi_{+}}(Y)$ and $\JTmdual{\Phi}(Y\cup\{x\})\leftarrow\JTmdual{\Phi_{+}}(Y)$
\item \label{JTf3step7} Return $\JTsp{[\JTmdual{\Phi}]}{\zeta}$
\end{enumerate}
\end{algorithm}

\noindent\hfil\rule{1\textwidth}{.4pt}\hfil

We now describe the function {\bf marginalise}:
\begin{itemize}
\item The function takes, as input, a sparse format $\JTsp{[\JTmdual{\Gamma}]}{\zeta}$ of the m-dual of a potential $\Gamma$ as well as a set $\{D_i:i\in\JTnatup{k}\})$ where $\zeta=\bigcup_{i=1}^k\JTpower{D_i}$.
\item The function outputs the set of potentials $\{\JTmdual{\Psi_i}:i\in\JTnatup{k}\}$ where, for every $i\in\JTnatup{k}$, $\Psi_i:=\JTmarg{\Gamma}{U_i}$.
\item The algorithm can be implemented to take a time of $\Theta\left(\sum_{i=1}^k\JTexp{|D_i|}\right)$ and to require $\Theta(|\zeta|)$ auxiliary space.
\item The correctness of the algorithm comes directly from Theorem \ref{JTtheorem22}
\end{itemize}

\begin{algorithm} \label{JTf4}
{\bf marginalise}$(\JTsp{[\JTmdual{\Gamma}]}{\zeta}, \{D_i:i\in\JTnatup{k}\})$:
\begin{enumerate}
\item For every $Z\in\zeta$ perform the following:
\newline
For every $i\in\JTnatup{k}$ with $Z\subseteq D_i$ set $\JTmdual{\Psi_i}(Z)\leftarrow\JTmdual{\Gamma}(Z)$
\item Return $\{\JTmdual{\Psi_i}:i\in\JTnatup{k}\}$
\end{enumerate}
\end{algorithm}

\noindent\hfil\rule{1\textwidth}{.4pt}\hfil

We now describe the recursive function {\bf transform3}:
\begin{itemize}
\item The function that takes as input the m-dual, $\JTmdual{\Phi}$, of a potential $\Phi$.
\item The function outputs the potential $\Phi$.
\item The algorithm can be implemented to take a time of $\Theta\left(|\JTunderlying{\Phi}|\JTexp{|\JTunderlying{\Phi}|}\right)$ and to require  $\Theta\left(\JTexp{|\JTunderlying{\Phi}|}\right)$ auxiliary space. This is proved immediately by induction over $|\JTunderlying{\Phi}|$
\item The correctness of the algorithm comes directly from Theorem \ref{JTtheorem21}, using induction over $|\JTunderlying{\Phi}|$.
\end{itemize}

\begin{algorithm} \label{JTf5}
{\bf transform3}$(\JTmdual{\Phi})$:
\newline
If $\JTunderlying{\JTmdual{\Phi}}=\emptyset$ then return $\Phi$. Else, perform the following:
\begin{enumerate}
\item \label{JTf5step1} Choose $x\in\JTunderlying{\Phi}$.
\item \label{JTf5step2} For each $Y\in\JTpower{\JTunderlying{\Phi}\setminus\{x\}}$ set $\JTmdual{\Phi_{-}}(Y)\leftarrow\JTmdual{\Phi}(Y)-\JTmdual{\Phi}(Y\cup\{x\})$ and $\JTmdual{\Phi_{+}}(Y)\leftarrow\JTmdual{\Phi}(Y\cup\{x\})$. Note that $\JTmdual{\Phi_{-}}$ and $\JTmdual{\Phi_{+}}$ are now m-duals of potentials in $\JTpots{\JTunderlying{\Phi}\setminus\{x\}}$.
\item \label{JTf5step3} Set $\Phi_{-}\leftarrow${\bf transform3}$(\JTmdual{\Phi_{-}})$ and ${\Phi_{+}}\leftarrow${\bf transform3}$(\JTmdual{\Phi_{+}})$.
\item \label{JTf5step4} For each $Y\in\JTpower{\JTunderlying{\Phi}\setminus\{x\}}$ set ${\Phi}(Y)\leftarrow{\Phi_{-}}(Y)$ and ${\Phi}(Y\cup\{x\})\leftarrow{{\Phi_{+}}(Y)}$
\item \label{JTf5step5} Return ${\Phi}$.
\end{enumerate}
\end{algorithm}

\noindent\hfil\rule{1\textwidth}{.4pt}\hfil

\subsection{Performing Operation \ref{JToper1}}
As stated at the start of the section, the only difference between ARCH-1 and ARCH-2 is the way that Operation \ref{JToper1}, used to do steps \ref{JTnewfirst} and \ref{JTsecond} of Algorithm \ref{JTalgorithm2second}, is performed.\\ In this subsection let $C$, $k$, $D_i$, $\Upsilon_i$, $\Gamma$ and $\Psi_i$ be as in the description of Operation \ref{JToper1}. That is: $C$ is a set. $D_1, ..., D_k$ are subsets of $C$ with $\bigcup_{i=1}^kD_i=C$. $\Upsilon_i$ is a potential in $\JTpots{D_i}$. $\Gamma:=\prod_{i=1}^k\Upsilon_i$ and $\Psi_i:=\JTmarg{\Gamma}{D_i}$. The goal of Operation \ref{JToper1} is to compute $\Psi_i$ for every $i\in\JTnatup{k}$.
\\
 ARCH-2 performs Operation \ref{JToper1} via the following algorithm:

\begin{algorithm}
\label{JTarchalgo}
Performing operation \ref{JToper1}:
\begin{enumerate}
\item \label{JTline1} For every $i\in\JTnatup{k}$ set $\JTdual{\Upsilon_i}\leftarrow{\bf transform1}(\Upsilon_i)$
\item \label{JTline2} Set $\JTsp{[\JTdual{\Gamma}]}{\zeta}\leftarrow{\bf product}(\{\JTdual{\Upsilon_i}:i\in\JTnatup{k}\})$
\item \label{JTline3} Set $\JTsp{[\JTmdual{\Gamma]}}{\zeta}\leftarrow{\bf transform2}(\JTsp{[\JTdual{\Gamma}]}{\zeta})$
\item \label{JTline4} Set $\{\JTmdual{\Psi_i}:i\in\JTnatup{k}\}\leftarrow{\bf marginalise}(\JTsp{[\JTmdual{\Gamma}]}{\zeta}, \{D_i:i\in\JTnatup{k}\})$
\item \label{JTline5} For every $i\in\JTnatup{k}$ set ${\Psi_i}\leftarrow{\bf transform3}(\JTmdual{\Psi_i})$ and return $\Psi_i$
\end{enumerate}
\end{algorithm}

\noindent\hfil\rule{1\textwidth}{.4pt}\hfil

To summarise, Algorithm \ref{JTarchalgo} does the following: First the potentials $\Upsilon_i$ are converted into their p-duals. From these p-duals, the p-dual of the product, $\Gamma$, of the potentials $\Upsilon_i$ is computed and stored in a sparse format. From this potential, a sparse format of the m-dual of $\Gamma$ is computed and is then used to compute the m-duals of the $D_i$-marginals, $\Psi_i$, of $\Gamma$. These m-duals are then converted into the potentials $\Psi_i$.

\noindent\hfil\rule{1\textwidth}{.4pt}\hfil

The correctness of the Algorithm \ref{JTarchalgo} is proved as follows: lines \ref{JTline1} and \ref{JTline5} are cleary valid by the descriptions of the functions {\bf transform1} and {\bf transform3}. Since $\Gamma=\prod_{i=1}^k\Upsilon_i$ line \ref{JTline2} is valid. Note that, since by line \ref{JTline2}, $\zeta=\bigcup_{i=1}^k\JTpower{\JTunderlying{\JTdual{\Upsilon_i}}}$, $\zeta$ is a straddle set. Hence, since by line \ref{JTline2} it is true that for all $Y\in\JTpower{C}$ with $Y\notin{\zeta}$ we have $\JTdual{\Gamma}(Y)=1$, line \ref{JTline3} is valid. Since, by line \ref{JTline2} we have $\zeta=\bigcup_{i=1}^k\JTpower{\JTunderlying{\JTdual{\Upsilon_i}}}=\bigcup_{i=1}^k\JTpower{D_i}$, line \ref{JTline4} is valid.

\noindent\hfil\rule{1\textwidth}{.4pt}\hfil

\subsection{Time and Space Complexity}

We now derive the time complexity of Algorithm \ref{JTarchalgo} and use it to calculate the time complexity of ARCH-2:
\newline
Lines \ref{JTline1} and \ref{JTline5} of Algorithm \ref{JTarchalgo} take a time of $\Theta\left(\sum_{i=1}^k|D_i|\JTexp{|D_i|}\right)$. Lines \ref{JTline2} and \ref{JTline4} take a time of \\$\mathcal{O}\left(\JTexp{|C|}+\sum_{i=1}^k|D_i|\JTexp{|D_i|}\right)$. 
Line \ref{JTline3} takes a time of $\mathcal{O}\left(|C|\JTexp{|C|}\right)$. The total time complexity of Algorithm \ref{JTarchalgo} is hence $\mathcal{O}\left(|C|\JTexp{|C|}+\sum_{i=1}^k|D_i|\JTexp{|D_i|}\right)$
\newline
Hence Equation \ref{JTfirst} and Step \ref{JTsecond} of algorithm \ref{JTalgorithm2second} both take a time of:
\begin{align}
&\mathcal{O}\left(| C|\JTexp{| C|}+\left(\sum_{ H\in\JTneighbour{ C}}| H\cap C|\JTexp{| H\cap{ C}|}\right)+\left(\sum_{\Lambda\in\JTfactor{ C}}|\Lambda|\JTexp{|\Lambda|}\right)\right)\\
=&\mathcal{O}\left(| C|\JTexp{| C|}+\left(|\JTpar{ C}\cap C|\JTexp{|\JTpar{ C}\cap{ C}|}+\sum_{ H\in\JTchildren{ C}}| H\cap C|\JTexp{| H\cap{ C}|}\right)+\left(\sum_{\Lambda\in\JTfactor{ C}}|\Lambda|\JTexp{|\Lambda|}\right)\right)\\
\subseteq&\mathcal{O}\left(| C|\JTexp{| C|}+\left(| C|\JTexp{|{ C}|}+\sum_{ H\in\JTchildren{ C}}| H\cap C|\JTexp{| H\cap{ C}|}\right)+\left(\sum_{\Lambda\in\JTfactor{ C}}|\Lambda|\JTexp{|\Lambda|}\right)\right)\\
=&\mathcal{O}\left(| C|\JTexp{| C|}+\left(\sum_{ H\in\JTchildren{ C}}| H\cap C|\JTexp{| H\cap{ C}|}\right)+\left(\sum_{\Lambda\in\JTfactor{ C}}|\Lambda|\JTexp{|\Lambda|}\right)\right)\\
\label{JTfactorequation}\subseteq&\mathcal{O}\left(| C|\JTexp{| C|}+\left(\sum_{ H\in\JTchildren{ C}}| H|\JTexp{| H|}\right)+\left(\sum_{\Lambda\in\JTfactor{ C}}|\Lambda|\JTexp{|\Lambda|}\right)\right)
\end{align}
We call this time complexity the ``computation time at $ C$". Note that every vertex $ H$ contributes $\mathcal{O}(| H|\JTexp{| H|})$ to the computation time at $ H$, a time of $\mathcal{O}(| H|\JTexp{| H|})$ to the computation time at $\JTpar{ H}$ (if it exists), and no time to computation time at any other vertex. Each vertex $ H$ hence contributes a total time of $\mathcal{O}(| H|\JTexp{| H|})$ to the running time of ARCH-2. In addition, by Equation \ref{JTfactorequation} we have that each factor $\Lambda$ contributes a time of $\mathcal{O}\left(|\JTunderlying{\Lambda}|\JTexp{|\JTunderlying{\Lambda}|}\right)$ to the running time of ARCH-2.
\newline
Note that the total running time of ARCH-2 is, up to a logarithmic factor, no worse than that of ARCH-1 (and Hugin propagation), and in cases where we have large vertices of high degree ARCH-2 is much faster than ARCH-1 (and Hugin propagation).

\noindent\hfil\rule{1\textwidth}{.4pt}\hfil
We now derive the space complexity of ARCH-2:

The auxiliary space requirement of Algorithm \ref{JTarchalgo} is the maximum space required by any of the functions which is $\mathcal{O}\left(|C||\zeta|\right)\subseteq\mathcal{O}\left(|C|\bigcup_{i=1}^k\JTexp{|D_i|}\right)$. Hence equation \ref{JTfirst} and step \ref{JTsecond} of algorithm \ref{JTalgorithm2second} both require an auxiliary space of $\mathcal{O}\left(| C|\left(\left(\sum_{ H\in\JTneighbour{ C}}\JTexp{| H\cap C|}\right)+\left(\sum_{\Lambda\in\JTfactor{ C}}\JTexp{|\JTunderlying{ C}}\right)\right)\right)$. This implies that, in addition to storing the messages and factors, ARCH-2 has a space requirement of only $\mathcal{O}\left(\operatorname{max}_{ C\in\JTvertex{\JT}}| C|\left(\left(\sum_{ H\in\JTneighbour{ C}}\JTexp{| H\cap C|}\right)+\left(\sum_{\Lambda\in\JTfactor{ C}}\JTexp{|\JTunderlying{ \Lambda}|}\right)\right)\right)$. Hence, the space requirement of ARCH-2 is not greater than that of ARCH-1 (and Shafer-Shenoy propagation) by more than a factor that is linear in width of the junction tree (and since this factor is logarithmic in the time complexity of the algorithm it is negligible).

\noindent\hfil\rule{1\textwidth}{.4pt}\hfil

\section{Incorporating Zeros}
\label{JTsection7}

So far we have only considered potentials that map into the positive reals. We now show how to generalise so that the codomain of a potential can be $\mathbb{R}^+\cup\{0\}$. In \cite{JTref1} the concept of a \textit{zero-concious number} is introduced to do this with Hugin propagation (for a wider range of queries). However, to work with ARCH-2 we need a slightly different object:
\begin{definition}{\bf MZC (multi-zero conscious) number:}
\begin{itemize}
\item An MZC number is a pair $(a, i)\in\mathbb{R}^+\times\mathbb{Z}$.
\item The product, $(a,i)\times (b,j)$,  of two MZC numbers, $(a,i)$ and $(b,j)$, is defined to be equal to $(c,k)$ where $c=ab$ and $k=i+j$.
\item For two MZC numbers $(a,i)$ and $(b,j)$ we define $(a,i)/(b,j):=(a,i)\times (1/b, -j)$.
\item The sum, $(a,i)+(b,j)$,  of two MZC numbers, $(a,i)$ and $(b,j)$, is defined to be equal to $(c,k)$ where $c$ and $k$ are defined are follows: If $i<j$ then $c:=a$ and $k:=i$, if $i=j$ then $c:=a+b$ and $k:=i$, and if $i>j$ then $c:=b$ and $k:=j$.
 \end{itemize}
 \end{definition}

\noindent\hfil\rule{1\textwidth}{.4pt}\hfil

A real number $x\in\mathbb{R}\cup\{0\}$ is converted into an MZC number as follows: If $x=0$ then it is converted into the MZC number $(1,1)$. Else it is converted into the MZC number $(x,0)$.

An MZC number $(a,i)$ is converted into a real number as follows: If $i\neq 0$ then it is converted into $0$. Else it is converted into $a$.

\noindent\hfil\rule{1\textwidth}{.4pt}\hfil

Zeros are incorporated into ARCH-2 as follows: Before running Algorithm \ref{JTarchalgo} all numbers (that is: the quantities $\Upsilon_i(Z)$) are converted from reals numbers into MZC numbers. Lines \ref{JTline1} to  \ref{JTline3} of Algorithm \ref{JTarchalgo} are then run with MZC numbers instead of reals. After Line \ref{JTline3} is complete then all MZC numbers (that is: the quantities $\JTmdual{\Gamma}(Z)$) are coverted from MZC numbers to real numbers. Lines \ref{JTline4} and \ref{JTline5} of Algorithm \ref{JTarchalgo} are then run using real numbers.

Due to the equivalence of ARCH-1/ARCH-2 to Hugin propagation we can, in Line \ref{JTdivision} of Algorithm \ref{JTalgorithm2second}, define division of a real number by zero to be equal to zero (or any other number) as is done in Hugin propagation (see \cite{JTref1}).

\noindent\hfil\rule{1\textwidth}{.4pt}\hfil

\section{Implementation Details}
\label{JTsection8}
In this section we make use of tree-structured data-structures. Whenever we use the word ``vertex" or ``leaf" we are referring to a vertex in one of these tree-structured data-structures (not a junction tree).
\newline
\newline
In this section we assume, without loss of generality, that $S=\JTnatup{n}$ for some $n\in\mathbb{N}$. Throughout the entire junction tree algorithm we maintain an array $\mathcal{A}$ of size $n$ such that each element of $\mathcal{A}$ (is a pointer to) a set (of pointers to) internal vertices of trees. Note that in our pseudo-code we will regard each element of $\mathcal{A}$ to be a set of vertices rather than a pointer to a set of pointers to vertices. 
We denote the $e$-th element of $\mathcal{A}$ by $\JTarray{e}$.
We also maintain a set $\JTL$ of (pointers to) leaves of trees. Like $\mathcal{A}$ we shall, in our pseudo-code, regard $\JTL$ as a set of leaves rather than a set of pointers to leaves.

$\mathcal{A}$ and $\JTL$ are used only for synchronised-searches and full-searches (see later) and in between different synchronised-searches/full-searches we have $\JTL=\emptyset$ and $\mathcal{A}(e)=\emptyset$ for all $e\in\JTnatup{n}$.

\noindent\hfil\rule{1\textwidth}{.4pt}\hfil
\subsection{Data-Structures}

An \textit{oriented binary tree} is a rooted tree in which every internal vertex $v$ has two children: one child is called the \textit{left-child} of $v$ and is denoted by $\JTlchild{v}$. The other is called the \textit{right-child} of $v$ and is denoted by $\JTrchild{v}$.
\newline
Given a vertex $v$ in an oriented binary tree we define $\JTancestor{v}$ to be the set of ancestors of $v$ (including $v$) and we define $\JTdescendants{v}$ to be the subtree of $v$ and its descendants.
\newline
\newline
Given a straddle-set $\zeta\subseteq\JTpower{\JTnatup{n}}$ we define the \textit{straddle-tree}, $\JTstree{\zeta}$ as follows: $\JTstree{\zeta}$ is an oriented binary tree who's internal vertices are labelled with numbers in $\bigcup{\zeta}$. Given an internal vertex $v$ we let $\JTlabel{v}$ be the label of $v$. The labels are such that given internal vertices $v$ and $w$ such that $w$ is a child of $v$ we have that $\JTlabel{w}>\JTlabel{v}$. We have a bijection, $\tau$, from the leaves of $\JTstree{\zeta}$ into the set $\zeta$ such that, for any leaf $l$ we have $\JTleafset{l}:=\{\JTlabel{v}: v\in\JTancestor{v},~ \JTrchild{v}\in\JTancestor{l}\}$.
\newline
Given a straddle-tree $\JTstree{\zeta}$ we will also refer to the tree that $\JTstree{\zeta}$ is based on by $\JTstree{\zeta}$
\newline
Note that since a straddle-tree $\JTstree{\zeta}$ is a full binary tree with $|\zeta|$ leaves it has only $2|\zeta|-1$ vertices in total.
\newline
Note also that for some set $X\subseteq\JTnatup{n}$ the straddle-tree $\JTstree{\JTpower{X}}$ is a balanced oriented binary tree of height $|X|$ such that all internal vertices at depth $i$ are labeled with the $(i+1)$th smallest number in $X$.
\newline
\newline
Given a potential $\Phi$ and a straddle-set $\zeta\subseteq\JTpower{\JTunderlying{\Phi}}$ we define the \textit{info-tree}, $\JTitreetwo{\Phi}{\zeta}$ as the straddle-tree $\JTstree{\zeta}$ with a map $\psi$ from the leaves of $\JTstree{\zeta}$ into $\mathbb{R}^+$ such that, for any leaf $l$ we have $\JTvalue{l}:=\Phi(\JTleafset{l})$.
\newline
\newline
Any sparse format, $\JTsp{\Phi}{\zeta}$ is stored as the info-tree $\JTitreetwo{\Phi}{\zeta}$. Any potential $\Phi$ (not in sparse format) is stored as the info-tree $\JTitreetwo{\Phi}{\JTpower{\JTunderlying{\Phi}}}$ which we shall denote by $\JTitreeone{\Phi}$

\noindent\hfil\rule{1\textwidth}{.4pt}\hfil

\subsection{Searches}

In this section we describe the ways that the algorithms perform efficient, simultaneous searches over straddle-trees. There are two types of simultaneous search we describe: \textit{full-searches} which are used in ARCH-1 and \textit{synchronised-searches} which are used in ARCH-2. We start by defining a \textit{ghost-search} which is what the simultaneous searches are based on.

\noindent\hfil\rule{1\textwidth}{.4pt}\hfil
{\bf Ghost-Search:} Given a set $X\subseteq\JTnatup{n}$, a \textit{ghost-search} of $X$ is the following algorithm, which is split up into a sequence of steps called \textit{time-steps}:
\newline
We maintain a stack $\JTstack$ such that each element of $\JTstack$ is either $\JTnull$ or of the form $(e,f)$ where $e\in X$ and $f\in\{1,2, 3\}$. $\JTstack$ is initialised to contain $(\operatorname{min}(X), 1)$ as a single element. On each time-step we do the following:
\newline
If the top element of $\JTstack$ is $\JTnull$ then remove it from $\JTstack$ which completes the time-step. Else, the top element of $\JTstack$ is $(e,f)$ for some $e\in X$ and $f\in\{1,2,3\}$ so we have the following cases:
\begin{enumerate}
\item $f=1$: In this case we remove $(e,f)$ from $\JTstack$ and then place $(e,2)$ on the top of $\JTstack$. If $e=\operatorname{max}(X)$ then next place $\JTnull$ on the top of $\JTstack$. Else place $(\operatorname{min}\{b\in X: b>e\}, 1)$ on the top of $\JTstack$. This completes the time-step.
\item $f=2$: In this case we remove $(e,f)$ from $\JTstack$ and place $(e,3)$ on the top of $\JTstack$. If $e=\operatorname{max}(X)$ then next place $\JTnull$ on the top of $\JTstack$. Else place $(\operatorname{min}\{b\in X: b>e\}, 1)$ on the top of $\JTstack$. This completes the time-step.
\item $f=3$: In this case we remove $(e,f)$ from $\JTstack$ which completes the time-step.
\end{enumerate}
The algorithm terminates when $\JTstack$ becomes empty.

\noindent\hfil\rule{1\textwidth}{.4pt}\hfil
We call a time-step in a ghost-search a \textit{leaf-step} if and only if at the start of the time-step we have that the top element of the stack, $\JTstack$, is $\JTnull$.
\newline
\newline
Note that a ghost search of a set $X$ simulates a depth-first search (in which, given an internal vertex $v$, $\JTdescendants{\JTlchild{v}}$ is explored before $\JTdescendants{\JTrchild{v}}$) of $\JTstree{\JTpower{X}}$ without having to store the whole tree in the memory. The time-steps in which (at the start of the time-step) $(e,f)$ is on the top of the stack corresponds to the times when the depth first search is at some internal vertex $v$ in $\JTstree{\JTpower{X}}$ with $\JTlabel{v}=e$ and it is the $f$-th time that we have encountered $v$ throughout the depth first search. The leaf-steps correspond to the times when the depth-first search is at the leaves of $\JTstree{\JTpower{X}}$. Hence, by the bijection $\tau$ (from the leaves of $\JTstree{\JTpower{X}}$ into $\JTpower{X}$), we have a one to one correspondence between the leaf-steps and the sets in $\JTpower{X}$.

\noindent\hfil\rule{1\textwidth}{.4pt}\hfil
{\bf Full-Search:} Given a multi-set $\JTmultiset$ of straddle-trees (or info-trees, as every info-tree has an underlying straddle-tree), a \textit{full-search} of $\JTmultiset$ is the following algorithm:
\newline
We first define $X$ to be the set of all labels, $\JTlabel{v}$, of the internal vertices, $v$, of the trees in $\JTmultiset$. Note that $X$ can be found and ordered quickly. Note that before running the synchronised-search the set $\JTL$ is empty and the array $\mathcal{A}$ has the empty set for every element (see the start of this section). Let $\mathcal{R}$ be the set of roots of the trees in $\JTmultiset$. We initialise by, for every $r\in\mathcal{R}$, adding $r$ to the set $\JTarray{\JTlabel{r}}$. After this initialisation we perform a ghost search of $X$. Let $\JTstack$ be the stack in the ghost search. At the end of every time-step in the ghost search we do the following:

\begin{enumerate}
\item If, at the start of the time-step, the top element of $\JTstack$ is $\JTnull$ (i.e. the time-step is a leaf-step) then we do nothing.
\item If, at the start of the time-step, the top element of $\JTstack$ is $(e,1)$ for some $e\in X$ then for every $v\in\JTarray{e}$ we do the following: If $\JTlchild{v}$ is a leaf then we add $\JTlchild{v}$ to $\JTL$. Else we add $\JTlchild{v}$ to $\JTarray{\JTlabel{\JTlchild{v}}}$.
\item If, at the start of the time-step, the top element of $\JTstack$ is $(e,2)$ for some $e\in X$ then for every $v\in\JTarray{e}$ we do the following: We first remove $\JTlchild{v}$ from $\JTarray{\JTlabel{\JTlchild{v}}}$. If $\JTrchild{v}$ is a leaf then we add $\JTrchild{v}$ to $\JTL$. Else we add $\JTrchild{v}$ to $\JTarray{\JTlabel{\JTrchild{v}}}$.
\item If, at the start of the time-step, the top element of $\JTstack$ is $(e,3)$ for some $e\in X$ then for every $v\in\JTarray{e}$ we remove $\JTrchild{v}$ from $\JTarray{\JTlabel{\JTrchild{v}}}$.
\end{enumerate}
Once the ghost search terminates, we set $\JTarray{\operatorname{min}(X)}\leftarrow\emptyset$ and then the full-search terminates. Note that upon termination of the full-search we have that $\JTL=\emptyset$ and for all $e\in\JTnatup{n}$ we have $\JTarray{e}=\emptyset$ as required.

\noindent\hfil\rule{1\textwidth}{.4pt}\hfil
A full-search of $\{\JTstree{\JTpower{Y_i}}: i\in\JTnatup{a}\}$ essentially does the following: Recall from above that given $X=\bigcup_{i=1}^aY_i$ there is a one to one correspondence between the leaf-steps and sets in $\JTpower{X}$. Suppose we are at a leaf-step. Let $Z$ be the set in $\JTpower{X}$ corresponding to the leaf-step. Then at the start of the leaf-step we have that $\JTL$ is equal to the set of leaves $l$ in the trees $\{\JTstree{\JTpower{Y_i}}: i\in\JTnatup{a}\}$ such that, given $l$ is a leaf of $\JTstree{\JTpower{Y_j}}$, we have $\tau(l)=Z\cup Y_j$.
\newline
\newline
Note that if $\{y_i:i\in\JTnatup{c}\}=\bigcup_{i=1}^aY_i$ with $y_i<y_j$ for all $i,j\in\JTnatup{c}$ with $i<j$ then a full-search of $\{\JTstree{\JTpower{Y_i}}: i\in\JTnatup{a}\}$ takes a time of $\Theta\left(\sum_{i=1}^c|\{j\in\JTnatup{a}:y_i\in Y_j\}|2^i\right)$ and that this is bounded above by $\mathcal{O}\left(a\JTexp{c}\right)$

\noindent\hfil\rule{1\textwidth}{.4pt}\hfil
{\bf Syncronised-Search:} Given a multi-set $\JTmultiset$ of straddle-trees (or info-trees, as every info-tree has an underlying straddle-tree), a \textit{synchronised-search} of $\JTmultiset$ is the following algorithm:
\newline
We first define $X$ to be the set of all labels, $\JTlabel{v}$, of the internal vertices, $v$, of the trees in $\JTmultiset$. Note that $X$ can be found and ordered quickly. Note that before running the synchronised-search the set $\JTL$ is empty and the array $\mathcal{A}$ has the empty set for every element (see the start of this section). Let $\mathcal{R}$ be the set of roots of the trees in $\JTmultiset$. We initialise by, for every $r\in\mathcal{R}$, adding $r$ to the set $\JTarray{\JTlabel{r}}$. After this initialisation we perform a ghost search of $X$. Let $\JTstack$ be the stack in the ghost search. At the end of every time-step in the ghost search we do the following:
\begin{enumerate}
\item If, at the start of the time-step, the top element of $\JTstack$ is $\JTnull$ (i.e. the time-step is a leaf-step) then we set $\mathcal{L}\leftarrow\emptyset$.
\item If, at the start of the time-step, the top element of $\JTstack$ is $(e,1)$ for some $e\in X$ then for every $v\in\JTarray{e}$ we do the following: If $\JTlchild{v}$ is a leaf then we add $\JTlchild{v}$ to $\JTL$. Else we add $\JTlchild{v}$ to $\JTarray{\JTlabel{\JTlchild{v}}}$.
\item If, at the start of the time-step, the top element of $\JTstack$ is $(e,2)$ for some $e\in X$ then for every $v\in\JTarray{e}$ we do the following: If $\JTrchild{v}$ is a leaf then we add $\JTrchild{v}$ to $\JTL$. Else we add $\JTrchild{v}$ to $\JTarray{\JTlabel{\JTrchild{v}}}$.
\item If, at the start of the time-step, the top element of $\JTstack$ is $(e,3)$ for some $e\in X$ then we set $\JTarray{e}\leftarrow\emptyset$.
\end{enumerate}
Once the ghost search terminates, the synchronised-search also terminates. Note that upon termination of the synchronised-search we have that $\JTL=\emptyset$ and for all $e\in\JTnatup{n}$ we have $\JTarray{e}=\emptyset$ as required.

\noindent\hfil\rule{1\textwidth}{.4pt}\hfil

A synchronised-search of $\{\JTstree{\zeta_i}:i\in\JTnatup{a}\}$  essentially does the following: Recall from above that given $X=\bigcup_{i=1}^a\left(\bigcup\zeta_i\right)$ 
there is a one to one correspondence between the leaf-steps and sets in $\JTpower{X}$. Suppose we are at a leaf-step. Let $Z$ be the set in $\JTpower{X}$ corresponding to the leaf-step. Then at the start of the leaf-step we have that $\JTL$ is equal to the set of leaves $l$ in the trees $\{\JTstree{\zeta_i}:i\in\JTnatup{a}\}$ such that $\tau(l)=Z$.
\newline
\newline
Note that a synchronised-search of $\{\JTstree{\zeta_i}:i\in\JTnatup{a}\}$ takes a time of $\mathcal{O}\left(\JTexp{|\bigcup_{i=1}^a\left(\bigcup\zeta_i\right)|}+\sum_{i=1}^a|\zeta_i|\right)$. However, often much of the ghost-search underlying the synchronised-search is unnecessary, meaning that the additive factor of $\mathcal{O}\left(\JTexp{|\bigcup_{i=1}^a\left(\bigcup\zeta_i\right)|}\right)$ can be reduced.

\noindent\hfil\rule{1\textwidth}{.4pt}\hfil

\subsection{Implementing Algorithm \ref{JTalgorithm1}}
\label{JTalgo1imp}
 In this subsection let $C$, $k$, $D_i$, $\Upsilon_i$, $\Gamma$ and $\Psi_i$ be as in the description of Operation \ref{JToper1}. That is: $C$ is a set. $D_1, ..., D_k$ are subsets of $C$ with $\bigcup_{i=1}^kD_i=C$. $\Upsilon_i$ is a potential in $\JTpots{D_i}$. $\Gamma:=\prod_{i=1}^k\Upsilon_i$ and $\Psi_i:=\JTmarg{\Gamma}{D_i}$. The goal of Operation \ref{JToper1} is to compute $\Psi_i$ for every $i\in\JTnatup{k}$.

\noindent\hfil\rule{1\textwidth}{.4pt}\hfil
We first describe a simple implementation of Algorithm \ref{JTalgorithm1} that takes a time of $\Theta\left(k\JTexp{|C|}\right)$:
\newline
\newline
We have, as input, the set, $\{\JTitreeone{\Upsilon_i}: i\in\JTnatup{k}\}$. Initially, for every $i\in\JTnatup{k}$ we set $\mathfrak{A}_i\leftarrow\JTstree{\JTpower{D_i}}$ (which is copied from $\JTitreeone{\Upsilon_i}$) and set $\JTvalue{l}\leftarrow 0$ for every leaf $l$ of $\mathfrak{A}_i$. We then do a full-search of $\{\JTitreeone{\Upsilon_i}: i\in\JTnatup{k}\}\cup\{\mathfrak{A}_i: i\in\JTnatup{k}\}$. At the start of every leaf-step in the full-search we do the following:
\newline
\newline
Let $U$ be the set of leaves in $\JTL$ that are in the trees $\{\JTitreeone{\Upsilon_i}: i\in\JTnatup{k}\}$ and let $W$ be the set of leaves in $\JTL$ that are in the trees $\{\mathfrak{A}_i: i\in\JTnatup{k}\}$. Set $\alpha\leftarrow\sum_{l\in U}\JTvalue{l}$ and then set, for every $l\in W$, $\JTvalue{l}\leftarrow\JTvalue{l}+\alpha$.
\newline
\newline
After the full-search has terminated we have $\mathfrak{A}_i=\JTitreeone{\Psi_i}$ for every $i\in\JTnatup{k}$.

\noindent\hfil\rule{1\textwidth}{.4pt}\hfil
We now describe how, by caching various quantities, Algorithm \ref{JTalgorithm1} can, while retaining the low space complexity, be sped up to take a time of only $\Theta\left(\sum_{i\in\JTnatup{|C|}}|\{j\in\JTnatup{k}: y_i\in D_j\}|\JTexp{i}\right)$ where $y_i$ is the $i$-th least element of $C$:
\newline
\newline
We have, as input, the set, $\{\JTitreeone{\Upsilon_i}: i\in\JTnatup{k}\}$. Initially, for every $i\in\JTnatup{k}$ we set $\mathfrak{A}_i\leftarrow\JTstree{\JTpower{D_i}}$ (which is copied from $\JTitreeone{\Upsilon_i}$) and set $\JTvalue{l}\leftarrow 0$ for every leaf $l$ of $\mathfrak{A}_i$. We then do a full-search of $\{\JTitreeone{\Upsilon_i}: i\in\JTnatup{k}\}\cup\{\mathfrak{A}_i: i\in\JTnatup{k}\}$. For all $i\in\JTnatup{k}$ let $l_i^t$ (resp. $q_i^t$) be the leaf of $\JTitreeone{\Upsilon_i}$ (resp. $\mathfrak{A}_i$) that is in $\JTL$ at the start of the $t$-th leaf-step in the full-search. During the full-search, in addition to maintaining the variable $\alpha$ we also maintain a variable $\beta$ as well as, for every $i\in\JTnatup{k}$, a variable $ \delta_i$. At the start of the first leaf step we do the following:
\begin{enumerate}
\item For all $i\in\JTnatup{k}$ set $ \delta_i\leftarrow 0$
\item Set $\alpha\leftarrow\prod_{i\in\JTnatup{|C|}}\JTvalue{l_i^1}$
\item Set $\beta\leftarrow\alpha$
\end{enumerate}
 At the start of the $t$-th leaf-step, for $t>1$, we do the following:
\begin{enumerate}
\item For all $i\in\JTnatup{k}$ such that $q_i^t\neq q_i^{t-1}$ set $\JTvalue{q^{t-1}_i}\leftarrow\JTvalue{q^{t-1}_i}+\beta- \delta_i$.
\item For all $i\in\JTnatup{k}$ such that $q_i^t\neq q_i^{t-1}$ set $ \delta_i\leftarrow\beta$
\item Set $Q\leftarrow\{i\in\JTnatup{k}: l_i^t\neq l_i^{t-1}\}$
\item Set $\alpha\leftarrow\alpha\prod_{i\in Q}(\JTvalue{p_i^t}/\JTvalue{p_i^{t-1}})$
\item Set $\beta\leftarrow\beta+\alpha$
\end{enumerate}
Once the full-search has terminated we set $\JTvalue{q_i^{|C|}}\leftarrow\JTvalue{q_i^{|C|}}+\beta- \delta_i$ for every $i\in\JTnatup{k}$. We then have $\mathfrak{A}_i=\JTitreeone{\Psi_i}$ for every $i\in\JTnatup{k}$.

\noindent\hfil\rule{1\textwidth}{.4pt}\hfil
Note that, in the above implementation we can first re-order $C$ in order to minimise the time. However, re-ordering $C$ means that we must re-construct the info-trees $\Upsilon_i$ to be consistent with the new order which takes a time of $\Theta\left(|D_i|\JTexp{|D_i|}\right)$ for every $i\in\JTnatup{k}$.
\newline
\newline
Note also that since the above implementation involves division we should use MZC-numbers (see section \ref{JTsection7}) instead of real numbers to avoid division by zero.

\noindent\hfil\rule{1\textwidth}{.4pt}\hfil
\subsection{Implementing the Functions of ARCH-2}
\label{JTFuncImpSection}

In this subsection we describe the implementation of the functions described in Section \ref{JTfunctions}. The notation of this subsection is as in the description of the algorithms in Section \ref{JTfunctions}
\newline
\newline
First note that in Step \ref{JTf1step1} of Algorithm \ref{JTf1}, Step \ref{JTf3step1} of Algorithm \ref{JTf3} and Step \ref{JTf5step1} of Algorithm \ref{JTf5} we are asked to choose $x\in\JTunderlying{\Phi}$. In these times we will always choose to set $x\leftarrow\operatorname{min}(\JTunderlying{\Phi})$.

\noindent\hfil\rule{1\textwidth}{.4pt}\hfil
Steps \ref{JTf3step2}, \ref{JTf3step3} and \ref{JTf3step4} of Algorithm \ref{JTf3} are performed together as follows:\newline\newline First define $r$ to be the root of $\JTitreetwo{\JTdual{\Phi}}{\zeta}$. Note that the straddle-tree underlying $\JTdescendants{\JTlchild{r}}$ is $\JTstree{\vartheta}$ so we can copy this tree and set $\mathfrak{A}_{-}\leftarrow\JTstree{\vartheta}$ and $\mathfrak{A}_{+}\leftarrow\JTstree{\vartheta}$. We then do a synchronised-search of $\{\mathfrak{A}_{-},  \JTdescendants{\JTlchild{r}}, \JTdescendants{\JTrchild{r}}\}$ (resp. $\{\mathfrak{A}_{+},  \JTdescendants{\JTlchild{r}}, \JTdescendants{\JTrchild{r}}\}$~). At the start of every leaf-step we do the following:\newline If $\JTL$ doesn't contain a leaf of $\mathfrak{A}_{-}$ (resp. $\mathfrak{A}_{+}$) we do nothing. Else, we have two cases:
\begin{enumerate}
\item $\JTL$ contains a leaf of $\JTdescendants{\JTrchild{r}}$: In this case we have $\JTL=\{l_0, l_1, l_2\}$ where $l_0$ is a leaf of $\mathfrak{A}_{-}$ (resp. $\mathfrak{A}_{+}$), $l_1$ is a leaf of $\JTdescendants{\JTlchild{r}}$ and $l_2$ is a leaf of  $\JTdescendants{\JTrchild{r}}$. We set $\JTvalue{l_0}\leftarrow\JTvalue{l_1}$ (resp. $\JTvalue{l_0}\leftarrow\JTvalue{l_1}/\JTvalue{l_2}$~)
\item $\JTL$ doesn't contain a leaf of $\JTdescendants{\JTrchild{r}}$: In this case we have $\JTL=\{l_0, l_1\}$ where $l_0$ is a leaf of $\mathfrak{A}_{-}$ (resp. $\mathfrak{A}_{+}$) and $l_1$ is a leaf of $\JTdescendants{\JTlchild{r}}$. We set $\JTvalue{l_0}\leftarrow\JTvalue{l_1}$ (resp. $\JTvalue{l_0}\leftarrow\JTvalue{l_1}$~).
\end{enumerate}
After the synchronised searches we have $\mathfrak{A}_{-}=\JTitreetwo{\JTdual{\Phi_{-}}}{\vartheta}$ and $\mathfrak{A}_{+}=\JTitreetwo{\JTdual{\Phi_{+}}}{\vartheta}$
\newline
\newline
Step \ref{JTf1step2} of Algorithm \ref{JTf1} and Step \ref{JTf5step2} of Algorithm \ref{JTf5} are performed similarly (with $\JTitreeone{\Phi}$ or $\JTitreeone{\JTmdual{\Phi}}$ instead of $\JTitreetwo{\JTdual{\Phi}}{\zeta}$ and $\JTpower{\JTunderlying{\Phi}\setminus\{x\}}$ instead of $\vartheta$.)

\noindent\hfil\rule{1\textwidth}{.4pt}\hfil
Steps \ref{JTf3step6} and \ref{JTf3step7} of Algorithm \ref{JTf3} are performed together as follows:\newline First set $\mathfrak{A}\leftarrow\JTstree{\zeta}$ (copied from the input). Let $r$ be the root of $\mathfrak{A}$. Do a synchronised search of $\{\JTdescendants{\JTlchild{r}}, \JTitreetwo{\JTmdual{\Phi_{-}}}{\vartheta}, \JTitreetwo{\JTmdual{\Phi_{+}}}{\vartheta}\}$ (resp. $\{\JTdescendants{\JTrchild{r}}, \JTitreetwo{\JTmdual{\Phi_{-}}}{\vartheta}, \JTitreetwo{\JTmdual{\Phi_{+}}}{\vartheta}\}$~). At the start of every leaf-step we do the following:\newline If $\JTL$ doesn't contain a leaf of $\mathfrak{A}$ then we do nothing. Otherwise we have that $\JTL=\{l_0, l_1, l_2\}$ where $l_0$ is a leaf of $\mathfrak{A}$, $l_1$ is a leaf of $\JTitreetwo{\JTmdual{\Phi_{-}}}{\vartheta}$ and $l_2$ is a leaf $\JTitreetwo{\JTmdual{\Phi_{+}}}{\vartheta}$. In these cases we set $\JTvalue{l_0}\leftarrow\JTvalue{l_1}+\JTvalue{l_2}$ (resp. $\JTvalue{l_0}\leftarrow\JTvalue{l_2}$~). \newline After the synchronised searches we have $\mathfrak{A}=\JTitreetwo{\Phi}{\zeta}$
\newline
\newline
Step \ref{JTf1step4} of Algorithm \ref{JTf1} and Step \ref{JTf5step4} of Algorithm \ref{JTf5} are performed similarly (with $\JTpower{\JTunderlying{\Phi}}$ instead of $\zeta$).

\noindent\hfil\rule{1\textwidth}{.4pt}\hfil
Step \ref{JTf2step1} of Algorithm \ref{JTf2} requires us to construct $\JTstree{\zeta}$ where $\zeta:=\bigcup_{i=1}^k\{\JTpower{\JTunderlying{\JTdual{\Upsilon_i}}}\}$. We do this as follows:
\newline
We maintain (and grow) a subtree $\mathfrak{A}$ of $\JTstree{\zeta}$ initialised to contain the root, $r$, as a single vertex (with $\JTlabel{r}\leftarrow\operatorname{min}\bigcup\zeta$). At every point in the algorithm there is a single vertex of $\mathfrak{A}$ that is designated as the \textit{active vertex}. Also, at every point in the algorithm we say that the active vertex is either \textit{left-oriented} or \textit{right-oriented}. We initialise such that the vertex $r$ is the active vertex and is left-oriented. After this initialisation we do a synchronised-search of  $\{\JTitreeone{\JTdual{\Upsilon_i}}:i\in\JTnatup{k}\}$. Let $\JTstack$ be the stack used in the synchronised search. At the start of every time-step after the first we do the following:
\begin{enumerate}
\item If the top element of $\JTstack$ is $0$ (i.e. the time-step is a leaf-step) then do the following: If the active vertex $v$ is currently left-oriented then add a vertex $w$ to $\mathfrak{A}$ such that $w=\JTlchild{v}$. If the active vertex $v$ is currently right-oriented then add a vertex $w$ to $\mathfrak{A}$ such that $w=\JTrchild{v}$. It is the case that $w$ is a leaf of $\JTstree{\zeta}$.
\item If the top element of $\JTstack$ is $(i,1)$ for some $i\in\JTnatup{n}$ then do the following: If $\mathcal{A}(i)=\emptyset$ then do nothing. Otherwise, given that the active vertex is currently $v$ and is currently left-oriented (resp. right-oriented), we add a vertex $w$ to $\mathfrak{A}$ such that $w=\JTlchild{v}$ (resp. $w=\JTrchild{v}$) and set $\JTlabel{w}\leftarrow i$. We then make $w$ the active vertex and designate it as left-oriented.
\item If the top element of $\JTstack$ is $(i,2)$ for some $i\in\JTnatup{n}$ then do the following:  If $\mathcal{A}(i)=\emptyset$ then do nothing. Otherwise, given that the active vertex is currently $w$, we keep $w$ as the active vertex but now designate it as right-oriented.
\item If the top element of $\JTstack$ is $(i,3)$ for some $i\in\JTnatup{n}$ then do the following: If $\mathcal{A}(i)=\emptyset$ then do nothing. Otherwise, given $v$ is currently the active vertex, we let the parent of $v$ become the active vertex.
\end{enumerate}
After the synchronised-search is complete we have that $\mathfrak{A}=\JTstree{\zeta}$

\noindent\hfil\rule{1\textwidth}{.4pt}\hfil
After we have constructed $\JTstree{\zeta}$, Step \ref{JTf2step2} of Algorithm \ref{JTf2} is performed as follows: Set $\mathfrak{A}\leftarrow\JTstree{\zeta}$. Do a synchronised-search of $\{\mathfrak{A}\}\cup\{\JTitreeone{\JTdual{\Upsilon_i}}:i\in\JTnatup{k}\}$. Whenever we are at the start of a leaf-step such that there exists a leaf, $l$, of $\mathfrak{A}$ in $\JTL$ we set $\JTvalue{l}\leftarrow\prod_{s\in\JTL\setminus\{l\}}\JTvalue{s}$. After the synchronised search we have that  $\mathfrak{A}=\JTitreetwo{\JTdual{\Gamma}}{\zeta}$

\noindent\hfil\rule{1\textwidth}{.4pt}\hfil
Algorithm \ref{JTf4} is implemented as follows: For every $i\in\JTnatup{k}$ set $\mathfrak{A}_i\leftarrow\JTstree{\JTpower{D_i}}$. Do a synchronised-search of $\{\JTitreetwo{\JTmdual{\Gamma}}{\zeta}\}\cup\{\mathfrak{A}_i:i\in\JTnatup{k}\}$. Whenever we are at the start of a leaf-step such that there exists a leaf, $l$, of $\JTitreetwo{\JTmdual{\Gamma}}{\zeta}$ in $\JTL$ we set, for every leaf $s\in\JTL\setminus\{l\}$, $\JTvalue{s}\leftarrow\JTvalue{l}$. After the synchronised-search we have that $\mathfrak{A}_i=\JTitreeone{\JTmdual{\Psi_i}}$

\noindent\hfil\rule{1\textwidth}{.4pt}\hfil
\section{Conclusion}
In this paper we reviewed the classic architectures of Shafer-Shenoy and Hugin propagation and then introduced two new junction tree architectures: The first, ARCH-1, has the speed (up to a constant factor) of Hugin propagation and the low space requirements of Shafer-Shenoy propagation. The second, ARCH-2, has space and time complexities (almost) at least as good as ARCH-1 and is significantly faster when we have large vertices of high degree in the junction tree. We first gave a high-level overview of the new architectures and then details of there efficient implementations.

\noindent\hfil\rule{1\textwidth}{.4pt}\hfil


\begin{thebibliography}{49}

\bibitem {JTref1}
J.D.~Park and A.~Darwiche.
\newblock Morphing the Hugin and Shenoy-Shafer Architectures.
\newblock In \emph{Proceedings of ECSQARU}, 2003, pages 149-160.

\bibitem {JTref2}
J.~Pearl.
\newblock Reverend Bayes on inference engines: A distributed hierarchical approach.
\newblock In \emph{Proceedings of the American Association of Artificial Intelligence National Conference on AI}, 1982, pages 133-136.

\bibitem {JTref3}
G.R.~Shafer and P.P.~Shenoy.
\newblock Probability Propagation.
\newblock In \emph{Annals of Mathematics and Artificial Intelligence}, volume 2, issue 1-4, pages 327-351. (1990)

\bibitem {JTref4}
V.~Lepar and P.P.~Shenoy.
\newblock A comparison of Lauritzen-Spiegelhalter, Hugin, and Shenoy-Shafer architectures for computing marginals of probability distributions.
\newblock In \emph{UAI'98 Proceedings of the Fourteenth conference on Uncertainty in artificial intelligence}, 1998, pages 328-337.

\bibitem {JTref5}
H.~Zu.\newblock An efficient implementation of belief function propagation.
\newblock In \emph{UAI'91 Proceedings of the Seventh conference on Uncertainty in Artificial Intelligencee}, 1991, pages 425-432.

\bibitem {JTref6}
P.P.~Shenoy.\newblock Binary Join Trees for Computing Marginals in the Shenoy-Shafer Architecture.
\newblock In \emph{International Journal of Approximate Reasoning}, volume 17, nos 2-3, pages 239-263. (1997)

\bibitem {JTref7}
D.~Smith and V.~Gogate.\newblock The inclusion-exclusion rule and its application to the junction tree algorithm.
\newblock In \emph{IJCAI '13 Proceedings of the Twenty-Third international joint conference on Artificial Intelligence}, 2013, pages 2568-2575.

\bibitem {JTref10}
F.V.~Jensen, S.~Lauritzen and K.~Olesen.\newblock Baysesian updating in recursive graphical models by local computation.
\newblock In \emph{Computational Statistics Quarterly}, volume 4, pages 269-282. (1990)

\bibitem {JTref11}
T.~Schmidt and P.P.~Shenoy. \newblock Some improvements to the Shenoy-Shafer and Hugin architectures for computing marginals.
\newblock In \emph{Artificial Intelligence}, volume 102, Issue 2, pages 323-333. (1998)

\end{thebibliography}
\end{document}